\documentclass{article}

\usepackage[preprint]{neurips_2025}

\usepackage{braket}

\usepackage[utf8]{inputenc} % allow utf-8 input
\usepackage[T1]{fontenc}    % use 8-bit T1 fonts
\usepackage{hyperref}       % hyperlinks
\usepackage{url}            % simple URL typesetting
\usepackage{booktabs}       % professional-quality tables
\usepackage{amsfonts}       % blackboard math symbols
\usepackage{nicefrac}       % compact symbols for \frac12, etc.
\usepackage{microtype}      % microtypography
\usepackage{lipsum}
\usepackage{fancyhdr}       % header
\usepackage{graphicx}       % graphics
\usepackage{xcolor}
\graphicspath{{media/}}     % organize your images and other figures under media/ folder
\usepackage{amsmath}
\usepackage{makecell}
\usepackage{multirow}
\usepackage{empheq} % For boxed equations
\usepackage{caption}

\usepackage{wrapfig}

\usepackage{algorithm}
\usepackage{algpseudocode}
\usepackage{amsfonts}
\usepackage{amssymb,amsmath,amsthm}
\usepackage{graphicx} 
\usepackage{adjustbox}
\usepackage{booktabs}
\title{Exact Solutions to the Quantum Schr\"odinger Bridge Problem}
\newtheorem{theorem}{Theorem}

\newtheorem{proposition}[theorem]{Proposition}
\newtheorem{definition}[theorem]{Definition}

\author{
  Mykola Bordyuh\thanks{This work was initiated at Pfizer and further developed at Bristol Myers Squibb, Discovery Biotherapeutics, Machine Learning Research group, Cambridge, MA, USA.} \\
  Machine Learning Research, Pfizer\\
  Bristol-Myers Squibb, ML Research \\
  Cambridge, Massachusetts, USA \\
  \texttt{mykola.bordyuh@bms.com} \\
\And
Djork-Arn\'e Clevert \\
Machine Learning Research\\ 
Pfizer Worldwide Research Development and Medical \\
Friedrichstraße 110, 10117 Berlin, Germany\\
\texttt{djork-arne.clevert@pfizer.com} \\
\And
Marco Bertolini \\
Machine Learning Research\\ 
Pfizer Worldwide Research Development and Medical \\
Friedrichstraße 110, 10117 Berlin, Germany\\
\texttt{marco.bertolini@pfizer.com}
}

\newcommand{\cN}{\mathcal{N}}
\newcommand{\bx}{\boldsymbol{x}}
\newcommand{\by}{\mathbf{y}}
\newcommand{\bOmegaQ}{\mathbb{Q}}
\newcommand{\bOmegaP}{\mathbb{P}}
\newcommand{\R}{\mathbb{R}}
\newcommand{\bW}{\mathbf{W}}
\newcommand{\bfb}{\mathbf{b}}
\newcommand{\dd}{\mathrm{d}}
\newcommand{\bv}{\mathbf{v}}
\newcommand{\bu}{\mathbf{u}}
\newcommand{\cL}{\mathcal{L}}
\newcommand{\p}{\partial}
\newcommand{\bSigma}{\mathbf{\Sigma}}
\newcommand{\bmu}{\boldsymbol{\mu}}
\newcommand{\tbx}{\widetilde{\bx}}

\newcommand{\ba}{\mathbf{a}}
\newcommand{\bOmega}{\boldsymbol{\Omega}}
\newcommand{\bPsi}{\boldsymbol{\Psi}}
\newcommand{\bX}{\boldsymbol{X}}

\newcommand{\bC}{\boldsymbol{C}}

\newcommand{\bK}{\boldsymbol{K}}

\newcommand{\bG}{\boldsymbol{G}}
\newcommand{\bI}{\mathbb{I}}

\newcommand{\psib}{\overline{\psi}}

\begin{document}

\maketitle

\begin{abstract}

The Quantum Schrödinger Bridge Problem (QSBP) describes the evolution of a stochastic process between two arbitrary probability distributions, where the dynamics are governed by the Schrödinger equation rather than by the traditional real-valued wave equation. Although the QSBP is known in the mathematical literature, we formulate it here from a Lagrangian perspective and derive its main features in a way that is particularly suited to generative modeling. We show that the resulting evolution equations involve the so-called Bohm (quantum) potential, representing a notion of non-locality in the stochastic process. This distinguishes the QSBP from classical stochastic dynamics and reflects a key characteristic typical of quantum mechanical systems.
In this work, we derive exact closed-form solutions for the QSBP between Gaussian distributions. Our derivation is based on solving the Fokker-Planck Equation (FPE) and the Hamilton-Jacobi Equation (HJE) arising from the Lagrangian formulation of dynamical Optimal Transport. We find that, similar to the classical Schrödinger Bridge Problem, the solution to the QSBP between Gaussians is again a Gaussian process; however, the evolution of the covariance differs due to quantum effects. Leveraging these explicit solutions, we present a modified algorithm based on a Gaussian Mixture Model framework, and demonstrate its effectiveness across several experimental settings, including single-cell evolution data, image generation, molecular translation and applications in Mean-Field Games.
\end{abstract}

\section{Introduction}

The \emph{Schrödinger Bridge Problem} (SBP), in its dynamical formulation via entropy-regularized optimal transport~\cite{cuturi2013sinkhorn}, seeks the most likely stochastic evolution that transports mass between two arbitrary distributions, $\pi_0$ and $\pi_1$, while remaining close (in a relative entropy sense) to a reference diffusion process such as a Wiener process~\cite{leonard2013survey, chen2021stochastic}. In machine learning, diffusion generative models can be viewed as a special case of the SBP, where the goal is to learn an optimal transformation from a simple reference distribution (typically Gaussian) to a complex target distribution~\cite{song2020score, ho2020denoising, de2021diffusion, vincent2011connection, hyvarinen2005estimation, sohl2015deep, huang2021variational}.
Under the transition density of Brownian motion, paths that transport mass from arbitrary $\pi_0$ to $\pi_1$ are highly unlikely; the SBP addresses this by selecting, among all such low-probability trajectories, the one that is most likely. This optimal stochastic process has time-marginal density $p(\mathbf{x}, t)$ that admits the factorization
$p(\mathbf{x}, t) = \phi(\mathbf{x}, t)\, \widehat{\phi}(\mathbf{x}, t)$, 
where the potentials $\phi$ and $\widehat{\phi}$ satisfy heat equations
\begin{align}
    \partial_t \phi &= -\beta \Delta \phi~, &
    \partial_t \widehat{\phi} &= \beta \Delta \widehat{\phi}~,
\end{align}
with $\beta$ denoting the diffusion (Wiener) coefficient.
This structure closely mirrors quantum mechanics, where the probability density is given by $p(\bx, t) = \psi(\bx, t) \psib(\bx, t)$, with $\psi$ satisfying the (imaginary-time) Schrödinger equation
$i \frac{\partial \psi}{\partial t} = -\beta \Delta \psi$, 
highlighting a deep mathematical connection between stochastic control and quantum dynamics~\cite{griffiths2019introduction, schrodinger1926undulatory}.
This analogy was rigorously formalized by Guerra and Morato~\cite{guerra1983quantization}, and extended in subsequent works~\cite{carlen1984conservative, nagasawa2012schrodinger, nelson2020quantum}, which demonstrated that the SBP admits an alternative formulation as a stochastic control problem, governed not by classical optimal transport dynamics~\cite{monge1781memoire, kantorovich1942translocation}, but by the kinematics of the Schrödinger equation.

In this work, we aim to bridge the conceptual and methodological gap between the ``classical'' SBP and the Guerra-Morato (GM) formulation, referred to as the Quantum Schrödinger Bridge Problem (QSBP), following the terminology of \cite{pavon2002quantum}. Our first contribution is to reformulate the GM Lagrangian in terms of a forward-backward stochastic differential equation (SDE) system and derive the necessary conditions that the associated quantities must satisfy. This reformulation enables the adaptation of various techniques originally developed for solving the classical SBP \cite{de2021diffusion, shi2023diffusion, chen2021likelihood, vargas2021solving, tong2023simulation, neklyudov2022action} to the quantum setting.
These methods typically involve learning forward and backward stochastic processes between the marginal distributions $\pi_0(\mathbf{x})$ and $\pi_1(\mathbf{x})$ using two neural networks, with a Iterative Proportional Fitting Procedure (IPFP) \cite{ruschendorf1995convergence, ireland1968contingency, essid2019traversing} optimization procedure. However, such approaches are often computationally intensive due to the necessity of sampling the whole trajectory, which can hinder their usability in practice \cite{shi2023diffusion}.
To address these limitations, a complementary line of research has focused on the analytical tractability of Gaussian distributions to develop more efficient SBP solvers \cite{chewi2020gradient, altschuler2021averaging, korotin2023light}. Exact solutions to the SBP are known in only a few cases, with the Gaussian setting being one of the most prominent~\cite{mallasto2022entropy, bunne2023schrodinger}. 

In this work, we extend the set of exact solutions by deriving the closed-form expression of the QSBP for Gaussian marginals. Our approach differs significantly from that of \cite{bunne2023schrodinger}, which employ tools from Riemannian geometry. In contrast, we adopt a Lagrangian framework and solve both the Fokker–Planck equation and the quantum Hamilton–Jacobi equation explicitly, offering a distinct theoretical perspective for addressing bridging problems with tractable distributions.

In short, this work places the \emph{quantum Schrödinger Bridge on equal footing with the classical SBP in both theoretical understanding and practical usability}. Specifically, we make the following contributions:
\begin{itemize}
\item In Section~\ref{s:classical2quantum}, we introduce the optimality conditions for the QSBP and show that they lead to the Quantum Hamilton-Jacobi Equation, whose solution is governed by the Schrödinger equation. While these results are not novel per se, we reformulate them in a manner that is suitable for generative score-matching models.
\item In Section~\ref{s:theorem}, we derive a closed-form solution to the QSBP for Gaussian measures.
\item In the subsequent sections, we extend the Gaussian Mixture Model (GMM) algorithm to our exact solution and apply it to a variety of bridging tasks, including image generation, single-cell data modeling and latent molecular properties translation. We also demonstrate an application to Mean Field Games via a variational formulation of the quantum Lagrangian.
\end{itemize}

\section{From Classical to Quantum Schrödinger Bridge}
\label{s:classical2quantum}

The goal of this section is to set up notation and motivate and formally introduce our problem setting. 

Let us denote by \textit{path measure} any positive measure $\mathbb{Q}\in M_+(\Sigma)$, where $\Sigma = C\left([0,1], X\right)$ is the space of all continuous paths with time-coordinate $t\in [0,1]$. \footnote{We are restricting our attention to the continuous case here. If we want to extend the definition to the discrete case as well, we need to replace $\Sigma = D\left([0,1], X\right)$ with the space of c\`adl\`ag (right-continuous left-limited) paths.} Let $P(\Sigma)$ be the space of probability measures on $\Sigma$, that is, $\bOmegaP\in M_+(\Sigma)$ and $\int_\Sigma d\bOmegaP = 1$. 
Given two (potentially unknown, but from which we can sample) distributions $\pi_0$ and $\pi_1$ on $X$ and a reference path measure $\mathbb{R}\in M_+(\Sigma)$, the Schrödinger Bridge Problem (SBP) aims at selecting a path measure $\mathbb{Q}\in M_+(\Sigma)$ such that it generates a mapping between the distributions, 
$\mathbb{Q}_{0,1}=\pi_{0,1}$, and it satisfies some notation of optimality, for instance that the relative entropy $H(\mathbb{Q}|\mathbb{P}) = \int_\Sigma \log\left(\frac{d\bOmegaQ}{d\bOmegaP}\right)d\bOmegaP$ is minimized. 
It is often the case that $\bOmegaP$ is a reversible Markov process, for instance a linear stochastic processes \cite{bunne2023schrodinger} or the reversible Brownian motion on $X=\R^n$ \cite{leonard2013survey}.
A useful construction of such a path measure $\bOmegaQ$ that descends from diffusion models is represented in terms of 
a solution to a stochastic differential equation (SDE)
\begin{align}
\label{eq:forwardSDE}
\mathrm{d}\bx(t) &= \bfb_+(\bx(t), t) \, \mathrm{d}t + \sqrt{2 \beta(t)} \, \mathrm{d}\bW(t)~,
&\bx(0)&\sim \pi_0~, \ \ \bx(1)\sim \pi_1
\end{align}
where $\bfb_+: \Sigma \rightarrow TX$ is the vector-valued drift coefficient. 
$\beta(t)$ is the diffusion coefficient, which we assume to be $\bx$-independent for simplicity, and $\mathrm{d} \bW$ is the reversible Brownian motion on $X$. 
As in generative score-matching, there exists a reverse-time SDE that generates the same marginal probability $p_t(\bx)$ as \eqref{eq:forwardSDE} for all $t\in [0,1]$ and is given by
\cite{nelson1966derivation, anderson1982reverse} 
\begin{align}
\mathrm{d}\bx(t) &= \left[\bfb_+(\bx, t) - 2\beta(t) \nabla \log p_t(\bx)\right]\dd t + \sqrt{2 \beta(t)}\mathrm{d}\bW(t) = \bfb_-(\bx, t)\mathrm{d}t + \sqrt{2 \beta(t)}\mathrm{d}\bW(t)~, \nonumber
\end{align}
with the same boundary conditions as in \eqref{eq:forwardSDE}. Parametrizing $\bOmegaQ$ with \eqref{eq:forwardSDE} and taking $\bOmegaP$ to be the reversible Brownian motion (i.e., \eqref{eq:forwardSDE} with $\bfb_+\equiv0$), the optimality condition is equivalent to a minimization of the kinetic energy of the process.
\begin{align}
\label{eq:kinetic}
\min_{p_t\in P(\Sigma), \bfb_+} \ & \frac12 \int \int |\bfb_+(\bx(t),t)|^2 \, p(\bx, t) \, \mathrm{d}t \, \mathrm{d}\bx ~,
\qquad\qquad \bx(0)\sim \pi_0~, \ \ \bx(1)\sim \pi_1~.
\end{align}
As pointed out in several works \cite{chen2015optimal, chen2023density, liu2022deep, liu2023generalized},
it is useful to extend the range of optimality conditions and to add a potential term $V(\bx(t))$ to the stochastic action, yielding the Generalized Schrödinger Bridge problem (GSBP), under which the process \eqref{eq:forwardSDE} satisfies
\begin{align}
\label{eq:GSBP}
\min_{p_t\in P(\Sigma), \bfb_+} \ & \int \int \left[ \frac12  |\bfb_+(\bx(t),t)|^2 + V(\bx(t)) \right] \, p(\bx, t) \, \mathrm{d}t \, \mathrm{d}\bx ~,
\quad \bx(0)\sim \pi_0~, \ \ \bx(1)\sim \pi_1~.
\end{align}
The inclusion of a potential term fundamentally modifies the nature of the optimal solution. The transport plan must now balance two competing factors: kinetic energy, which quantifies the "speed" at which the distribution $\pi_0$ is transported to $\pi_1$, and potential energy, which introduces additional costs in the path measure space, penalizing certain trajectories more heavily based on their interaction with the potential. Although the quantity $V(\bx)$ in \eqref{eq:GSBP} potentially depends on the coordinates $\bx$, it takes the role of an \textit{external potential}, as it simply adds a (point-dependent) cost in the space $\Sigma$, realizing a geometric or interaction-based bias.

Regardless of the specific Lagrangian chosen for the optimality of the process, 
the time evolution of the marginal density $p_t(\bx(t))$ is governed by the continuity equation (also
Fokker-Planck Equation (FPE) \cite{risken1996fokker})
\begin{align}
\label{eq:FPE}
    \p_t p_t(\bx(t)) &= -\nabla \cdot (\bfb_+(\bx, t) \ p_t(t))+ \beta(t) \Delta p_t(\bx(t)) = 
    - \nabla \cdot (\bv\,p)~,
\end{align}
where we introduced the drift and osmotic velocities \cite{nelson1966derivation}
\begin{align}
\label{eq:osmotic}
\bv &= \frac{\bfb_+ + \bfb_-}{2} = \bfb_+(\bx(t), t) -\beta(t)\nabla \log p_t(\bx(t))~,
&\bu &= \frac{\bfb_+ - \bfb_-}{2} = \beta \nabla \log p(\bx(t), t)~.
\end{align}

\subsection{The Guerra-Morato Lagrangian and the Quantum Schrödinger Bridge}

In this work, we generalize the concept of the potential in \eqref{eq:GSBP} and we consider it to be a function on the whole configuration space (parametrized by both $\bx$ and $\dot{\bx}$), $V = V(\bx(t), \dot{\bx}(t))$. Specifically, we set $ V(\bx(t), \dot{\bx}(t)) = \nabla \cdot \bfb_+(\bx(t), t)$.
This leads to a Lagrangian known as the 
\textit{Guerra-Morato Lagrangian} \cite{faris2014diffusion, guerra1983quantization}, 
Following \cite{pavon2002quantum, pavon2004footnote}, we denote the following problem the \textit{Quantum Schrödinger Bridge Problem} (QSBP):
\begin{definition}
\label{def:QSPB}
    The Quantum Schrödinger Bridge Problem is defined as the distribution matching process that minimizes the Lagrangian
    \begin{align}
        \label{eq:QSBP}
        % \min_{p_t\in P(\Sigma), \bfb_+} \ 
        \cL_{\text{QSB}} = & \int \int \left[ \frac12  |\bfb_+(\bx(t),t)|^2 + \nabla\cdot \bfb_+(\bx(t), t) \right] \, p(\bx, t) \, \mathrm{d}t \, \mathrm{d}\bx ~,
    \end{align}
    subject to 
    \begin{align}
        \mathrm{d}\bx(t) &= \bfb_+(\bx(t), t) \, \mathrm{d}t + \sqrt{2 \beta(t)} \, \mathrm{d}\bW(t)~,
&\bx(0)&\sim \pi_0~, \ \ \bx(1)\sim \pi_1
    \end{align}
\end{definition}
In the remainder of this section, we will show that the governing equation of the QSBP is the Schrödinger equation from quantum mechanics (hence the \textit{quantum} attribute) and derive some properties of its kinematics.
Notably, it can be related to the Madelung formulation of dynamics \cite{madelung1926anschauliche, wyatt2005quantum}.

\subsection{The Schrödinger equation and the Bohm potential}
\begin{figure}[t!]
    \centering
    \includegraphics[width=\textwidth]{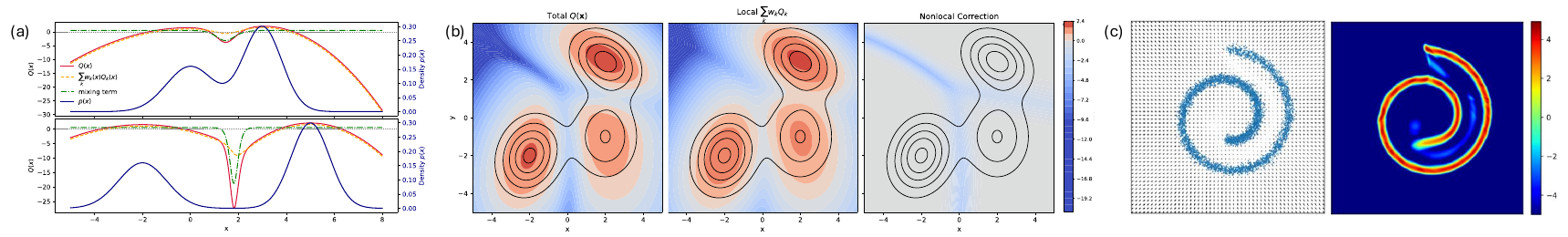}
    \caption{Example of 1d (a) and 2d distributions (b) with the corresponding Bohm potentials. (c) Learned scores of the data distribution $\nabla \log p(x)$ (top) and learned Bohm potential (bottom) \eqref{eq:bohm} for the Swiss roll dataset. The Bohm potential peaks at the data points and drops for points out of distribution ($Q(x)<-5 = -5$ is applied for visualization purposes).}
    \label{fig:bohm}
\end{figure}

In this section, we explore the solution of the QSBP and its connection to fundamental equations of physics and probability. Specifically, we derive the Hamilton-Jacobi equation that governs the dynamics of the system, enriched by the inclusion of the Bohm potential. This quantum potential introduces non-local effects, reflecting how each particle's behavior depends not only on local properties but also on the global configuration of the system. Further, we establish links between the optimal transport framework and quantum mechanics, as we show that the governing equation of motion is the Schrödinger equation.
All proofs can be found in Appendix \ref{app:proof_sec2}.
\begin{proposition}
The solution of the QSBP (Definition \ref{def:QSPB}) is described by the quantum Hamilton-Jacobi equation
\begin{align}
\label{eq:quantumHJ}
    \partial_t S(\bx) + \frac12|\nabla S(\bx)|^2 = -Q(\bx)~,
\end{align}
where $Q(\bx)$ is known as the Bohm potential (or quantum potential) and is given as
\begin{align}
\label{eq:bohm}
    Q(\bx) = -2 \beta(\bx)^2\frac{\Delta \sqrt{p_t(\bx)}}{\sqrt{p_t(\bx)}}= -\beta(t)^2 (\Delta \log p_t(\bx)+ \frac12 |\nabla\log p_t(\bx)|^2)~.
\end{align}
\end{proposition}
The quantum potential $Q$ represents a notion of non-locality: each particle evolving in the process do not merely perceives the local effects due to the local potentials, but is also affected by the information about the whole motion through $Q$. In Figure \ref{fig:bohm}c we depicted the Bohm potential for the learned density for the Swiss roll dataset. 
We note that the quantum potential $Q$ coincides exactly with the score matching objective \cite{hyvarinen2005estimation} up to the scaling coefficient $\beta^2$. 

Next, we show that the dynamics defined by \eqref{eq:QSBP} follows the solution of the Schrödinger equation. 
\begin{proposition}
The stationary points of the QSBP satisfy the time-dependent Schr{\"o}dinger equation
\begin{align}
\label{eq:schroedinger}
    i \p_t\psi(\bx, t) &= -\beta(t) \Delta \psi(\bx, t)~,   & \psi(\bx)  \psi^*(\bx) &= p_t(\bx)~,
\end{align}
where the wavefunction $\psi(\bx, t) = \sqrt{p_t(\bx)}e^{\frac{i}{2\beta(t)} S(\bx, t)}$ and phase $S$ are related to the drift velocity $\bv$ via a gradient $\bv = \nabla S$.
%(see appendix (\ref{sec: sch_variational}) for the derivation)
\end{proposition}
Here, the wavefunction $\psi = \sqrt{p(x)} e^{\frac{i}{2\beta} S}$ has a phase $S$ related to the drift velocity $\bv$ via the gradient $\bv = \nabla S$. While the Schr\"odinger wavefunction $\psi$ is complex \cite{le2021parameterized}, the heat equation functions 
associated with the SBP are real.
The Schr{\"o}dinger equation can thus be regarded as a heat equation in imaginary time. 

\begin{figure}[t!]
    \centering
\includegraphics[width=\textwidth]{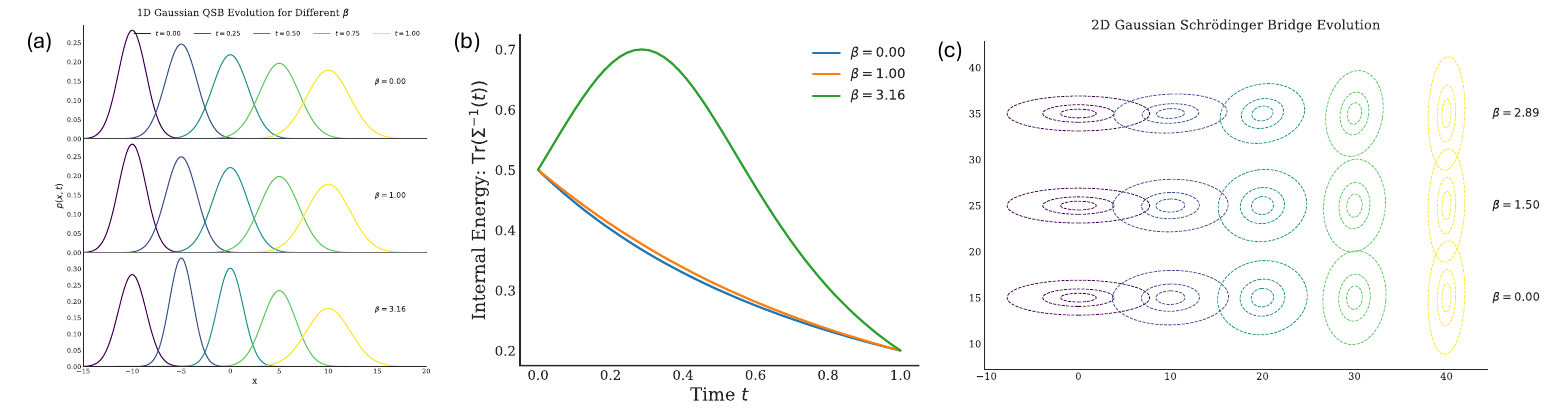}
    \caption{Visualization of Gaussian propagation: (a) 1d examples for different values of $\beta$ with relative total internal potential energy (b); (c) 2d examples for different values of $\beta$.} 
    \label{fig:gaussian_expansion}
\end{figure}

\subsection{The Bohm Potential for Gaussian Distributions and Internal Potential Energy}
\label{ss:bohm_gaussian}

Consider the multivariate Gaussian distribution $p(\bx, t)=\cN(\bx; \bmu, \bSigma)$ defined by
\begin{equation}
\label{eq:multi_gauss}
    p(\bx, t) = \frac{1}{(2\pi)^{\frac{n}2} \sqrt{\det \bSigma(t)}} \exp\!\left(-\frac12 (\bx - \bmu(t))^\top \bSigma^{-1}(t) (\bx - \bmu(t))\right)\,.
\end{equation}
A simple calculation (which we report in detail in Appendix \ref{app:bohm_gaussian}) yields the explicit expression for the Bohm potential \eqref{eq:bohm}
\begin{align}
    Q(\bx) = \beta(t)^2 \left[ \mathrm{Tr}\Bigl(\bSigma^{-1}(t)\Bigr) - \frac12 (\bx - \bmu(t))^\top \Bigl(\bSigma^{-1}(t)\Bigr)^2 (\bx - \bmu(t)) \right]~.
\end{align}
The total internal potential energy associated with a Gaussian distribution is given by the expectation value of $Q(\bx)$ under the distribution \eqref{eq:multi_gauss}
\begin{align}
    \int Q(\bx) p(\bx, t)  \mathrm{d}\bx &= \beta(t)^2 \mathrm{Tr}\Bigl(\bSigma^{-1}(t)\Bigr) - \frac{\beta(t)^2}{2} \int (\bx - \bmu(t))^\top \Bigl(\bSigma^{-1}(t)\Bigr)^2 (\bx - \bmu(t)) p(\bx, t) \mathrm{d}\bx \nonumber\\
    & =\frac{\beta(t)^2}{2}\,\mathrm{Tr}\Bigl(\bSigma^{-1}(t)\Bigr)\,.
\end{align}

\subsection{Bohm Potential of a Gaussian Mixture}

Since it will be relevant for the discussion below concerning our algorithm for a Gaussian Mixture Model, we also derive here the Bohm potential for Gaussian mixture distribution of the form
\begin{align}
p(\bx) &= \sum_{k=1}^K \alpha_k \,\cN(\bx; \bmu_k, \bSigma_k)~,
  &\alpha_k &\ge 0~,  &\sum_{k=1}^K \alpha_k &= 1~.
\end{align}
A slightly lengthy calculation (which we report in Appendix \ref{app:bohm_gaussian}) yields
\begin{align}
\label{eq:bohm_mixture}
Q(\bx)&=\sum_{k=1}^K w_k(\bx) Q_k(\bx) \nonumber\\
&\quad +\frac{\beta^2}{2} \sum_{k=1}^K w_k(\bx)(\bx - \bmu_k)^\top\bSigma_k^{-1}  \Bigl[
  \sum_{j=1}^K w_j(\bx)
\bSigma_j^{-1}(\bx - \bmu_j)-\bSigma_k^{-1}(\bx - \bmu_k)
\Bigr]~,
\end{align}
where we defined the posterior mixture weights
\begin{align}
w_k(\bx) &= \frac{\alpha_k\,\cN(\bx; \bmu_k, \bSigma_k)}{p(\bx)}~,
 &\sum_{k=1}^K w_k(\bx)& = 1~,
\end{align}
and $Q_k$ is the Bohm potential for the $k{}^{\text{th}}$-summand. 
This shows that the Bohm potential of a Gaussian mixture is given by a responsibility‐weighted sum of the single‐Gaussian Bohm potentials and an extra term that reflects the nontrivial mixture‐log‐density structure.

\newpage
\section{Closed-form Solution of the Quantum Schr\"odinger Bridge 
Problem for Multivariate Gaussian Measures}
\label{s:theorem}

\begin{table}[t!]
\caption{Comparison of Lagrangians and Solutions for different SBPs. We set $\widetilde{\Sigma}_{01}= \bSigma_0^{\frac12}\bSigma_1\bSigma_0^{\frac12}$.}
\label{table:combined_lagrangians}
\vskip 0.1in
\begin{small}
\begin{sc}
\begin{adjustbox}{width=1\textwidth}
\begin{tabular}{c|ccc}
\toprule
\textbf{Aspect} & \textbf{Benamou-Brenier OT} & \textbf{Classical SBP} & \textbf{Quantum SBP} \\ 
\midrule
\multirow{2}{*}{\makecell{\textbf{Feasibility} \\ \textbf{Parametrization}}}  & $\dd \bx = \bfb_+(\bx, t)\mathrm{d}t$ & $\dd \bx = \bfb_+(\bx, t) \mathrm{d}t + \sqrt{2 \beta(t)} \mathrm{d}\bW$ & $\dd \bx = \bfb_+(\bx, t) \mathrm{d}t + \sqrt{2 \beta(t)}  \mathrm{d}\bW$ \\
& $\bx(0)\sim \pi_0~, \ \bx(1)\sim \pi_1$ & $\bx(0)\sim \pi_0~, \ \bx(1)\sim \pi_1$ & $\bx(0)\sim \pi_0~, \ \bx(1)\sim \pi_1$ \\
\midrule
\makecell{\textbf{Optimality} \\ \textbf{Objective}}  & $\min \frac12 \mathbb{E} \int_0^1 \bfb_+^2 \mathrm{d} t$ & $\min \frac12 \mathbb{E} \int_0^1 \bfb_+^2 \, \mathrm{d} t$ & $\min \mathbb{E}  \int_0^1 \left( \frac{\bfb_+^2}{2} + \beta \, \nabla \cdot \bfb_+ \right)  \mathrm{d} t$ \\
\midrule
\multirow{2}{*}{\makecell{\textbf{Continuity} \\ \textbf{Equation}}} & $\partial_t p + \nabla \cdot (p\, \bv) = 0$ & $\partial_t p + \nabla \cdot (p\, \bv) = 0$ & $\partial_t p + \nabla \cdot (p \bv) = 0$ \\
& $\bv = \bfb_+$ &  $\bv  = \bfb_+ -\beta \nabla \log p_t(\bx)$ 
&  $\bv  = \bfb_+ -\beta \nabla \log p_t(\bx)$\\
\midrule
\multirow{2}{*}{\makecell{\textbf{Hamilton-Jacobi} \\ \textbf{Equation}}} & $\partial_t S + \frac12 |\nabla S|^2 = 0$ & $\partial_t S + \frac12 |\nabla S|^2 = -\beta \Delta S$ & $\partial_t S + \frac12 |\nabla S|^2 = -Q(\bx)$ \\ 
& $\bv = \nabla S$ & $\bfb_+ = \nabla S$ & $\bv = \nabla S$ \\
\midrule
\multirow{2}{*}{\makecell{\textbf{Evolution} \\ \textbf{Equations}}} & $\bx(t) = \bx(0) + \bv t$ & $\frac{\partial \phi}{\partial t} = -\beta \, \Delta \phi$ & $i \frac{\partial \psi}{\partial t} = -\beta \Delta \psi$ \\ 
& & $\frac{\partial \hat{\phi}}{\partial t} = \beta \, \Delta \hat{\phi}$ & $\psi = \sqrt{p} \exp{\left(\frac{i}{2 \beta} S\right)}$ \\
\midrule
\midrule
\makecell{\textbf{Mean} \\ \textbf{Evolution}} & 
$\bmu_0 + (\bmu_1-\bmu_0)t$ & $\bmu_0 + (\bmu_1-\bmu_0)t$ &$\bmu_0 + (\bmu_1-\bmu_0)t$ \\
\midrule
\makecell{\textbf{Variance} \\ \textbf{Evolution}} & {\tiny $\bSigma_0^{-\frac12}\Bigl[(1-t)\bSigma_0 + t  \widetilde{\bSigma}_{01} ^{\frac12}  \Bigr]^2\bSigma_0^{-\frac12}$ }
& {\tiny $\bSigma_0^{-\frac12}\Bigl[(1-t)\bSigma_0 + t  \Bigl( \widetilde{\bSigma}_{01} + \beta^2 \bI \Bigr)^{\frac12}  \Bigr]^2\bSigma_0^{-\frac12}- t \beta^2 \bSigma_0^{-1}$ }
& {\tiny $\bSigma_0^{-\frac12}\Bigl[(1-t)\bSigma_0 + t  \Bigl( \widetilde{\bSigma}_{01} - \beta^2 \bI \Bigr)^{\frac12}  \Bigr]^2\bSigma_0^{-\frac12}+ t \beta^2 \bSigma_0^{-1}$ }\\
\midrule
\bottomrule
\end{tabular}
\end{adjustbox}
\end{sc}
\end{small}
\end{table}

In this section, we discuss analytical solutions of the QSBP for the case of two Gaussian distributions, $\cN(\bmu_0, \bSigma_0)$ and $\cN(\bmu_1, \bSigma_1)$, where $\bx_{0,1}\in\R^n$ are the means, and $\bSigma_{0,1}\in \R^{n\times n}$ are the covariance matrices of the respective distributions.
Our main result is given by the following
\begin{theorem}
Given a probability distribution of the form
\begin{align}
    \label{eq:multi_gaussian_t}
    p(\bx, t) = \frac{1}{(2\pi)^{n/2} \sqrt{\det \bSigma(t)}} \exp\left(-\frac12 (\bx - \bmu(t))^\top \bSigma^{-1}(t) (\bx(t) - \bmu(t))\right)~,
\end{align}
it solves the QSBP \eqref{eq:QSBP} with boundary conditions 
$\pi_{0,1}(\bx) = \cN(\bx; \bmu_{0,1}, \bSigma_{0,1})$, 
where 
\begin{align}
\label{eq:gauss_solution}
\bmu(t)&=\bmu_0+(\bmu_1 - \bmu_0)t~,\nonumber\\
\bSigma(t) &= \bSigma_0^{-\frac12}\Bigl[(1-t)\bSigma_0 + t  \Bigl( \bSigma_0^{\frac12}\,\bSigma_1\,\bSigma_0^{\frac12} - \beta^2 \bI \Bigr)^{\frac12}  \Bigr]^2\bSigma_0^{-\frac12} \;+\;  \, t \beta^2 \bSigma_0^{-1}~.
\end{align}
\end{theorem}

We refer to appendix \ref{app:proof_thorem} for the full proof.
The existance of a solution is subject to the condition that the matrix 
$\bSigma_0^{\frac12}\,\bSigma_1\,\bSigma_0^{\frac12} - \beta^2 \bI$ is semi-positive-definite for all $t$, which implies that 
$\beta \leq \sqrt{\lambda_{\text{min}}}$ where $\lambda_{\text{min}}$ is the smallest eigenvalue of the matrix $\bSigma_0^{\frac12}\,\bSigma_1\,\bSigma_0^{\frac12}$.
The term $\bSigma_0^{\frac12}\,\bSigma_1\,\bSigma_0^{\frac12} - \beta^2 \bI$ is the multi-dimensional version of the Gaussian squeezing coefficient of quantum mechanics \cite{ford2002wave}. 
Finally, we note that the solution above reduces to the known solution of the Benamou-Brenier Optimal Transport Problem in the limit $\beta(t)=0$. 
In its formulation of the Optimal Mass Transport problem, the absence of stochastic effects leads to the objective to be the minimization of the kinetic energy. 
Indeed, setting $\beta=0$ in the QSBP solution \eqref{eq:gauss_solution} we obtain
% \begin{align}
{\small
$
\bSigma(t) = \bSigma_0^{-\frac12}\Bigl[(1-t)\bSigma_0 + t  \Bigl( \bSigma_0^{\frac12}\,\bSigma_1\,\bSigma_0^{\frac12} \Bigr)^{\frac12}  \Bigr]^2\bSigma_0^{-\frac12}~,
$}
in accordance with the known expression \cite{dowson1982frechet, onken2021ot}.
In table \ref{table:combined_lagrangians} we summarized the defining equations and the solutions for various bridging problem formulations. We note that the Benamou-Brenier OT problem is the $\beta=0$ limit
of both the ``classical'' SBP and the quantum SBP.
In Figure \ref{fig:gaussian_expansion}, we illustrate two examples of the dynamics governed by \eqref{eq:gauss_solution} for different values of $\beta$. In (a), we observe that for $\beta = 0$, the standard deviation evolves linearly throughout the bridging process. In contrast, for higher values of $\beta$, the intermediate Gaussian distributions undergo an initial squeezing phase before relaxing toward the target distribution. This behavior reflects the influence of the quantum potential, which is jointly minimized with the kinetic energy along the path.
Figure \ref{fig:gaussian_expansion}c presents a two-dimensional example. Here, we also observe non-local effects induced by the quantum regularization. While for $\beta = 0$ the covariance evolves independently along each dimension in a linear fashion, larger values of $\beta$ result in a more global deformation of the distribution’s shape. Notably, this includes a rotation-like effect, despite the absence of any explicit rotational term in the SDE parametrization \cite{bertolini2025generative}.

\section{Evolution of the Gaussian Mixture model: Algorithm}

\begin{minipage}[t]{0.5\textwidth}
Building on the analytical solution to the QSBP, we extend the Gaussian Mixture Model (GMM)-Evolution algorithm \cite{korotin2023light} to learn a bridge between the data distributions $\pi_0(\bx)$ and $\pi_1(\bx)$. Specifically, we define a process mediated by a mixture of Gaussian distributions of the form: \begin{align} p(\bx,t) = \sum_k \alpha_k \cN\left(\bx; \bmu_k(t), \bSigma_k(t)\right)~. \end{align} At time $t=0$, we fit the initial density using a GMM, and allow each Gaussian component to evolve independently over time. 
The evolved probability distribution is then fitted to $\pi_1(\bx)$, ensuring a consistent representation of the wavepacket dynamics. These steps are repeated iteratively until convergence. At the end of the procedure, we obtained the GMM parameters that best fit both $\pi_0$ at $t=0$ and $\pi_1$ at $t=1$.
\end{minipage}%
\hfill%
\begin{minipage}[t]{0.47\textwidth}
\vspace{-0.65cm}
\begin{algorithm}[H]
\caption{GMM Wavepacket Propagation}
\label{alg: gmm}
\begin{algorithmic}[1]
\State \textbf{Initialize:}
\State $N_{\text{Gaussians}}$ \Comment{Number of GMM components}
\State $\theta_0 = \{\alpha_k, \bmu_k(0), \bSigma_k(0)\}, \quad k=1:N$ \Comment{Initial GMM parameters}
\Repeat
\State Sample $n$ points $\{x_0^{(i)}\}_{i=1}^{n}$ from $\pi_0(\bx)$
\State Fit GMM parameters $\theta_0$ to $\{x_0^{(i)}\}_{i=1}^{n}$
\State Propagate wavepackets via \eqref{eq:gauss_solution}:
\State $\theta_1 = \text{Propagate}(\theta_0)$
\State Sample $n$ points $\{x_1^{(i)}\}_{i=1}^{n}$ from $\pi_1(\bx)$
\State Fit GMM parameters $\theta_1$ to $\{x_1^{(i)}\}_{i=1}^{n}$
\Until{convergence}
\State \textbf{Output:}  $\theta^\star = \{\alpha_k, \bmu_k(t), \bSigma_k(t)\}_{k=1}^N$
\end{algorithmic}
\end{algorithm}
\end{minipage}

\begin{wraptable}{r}{0.45\textwidth}
\includegraphics[width=0.45\textwidth]{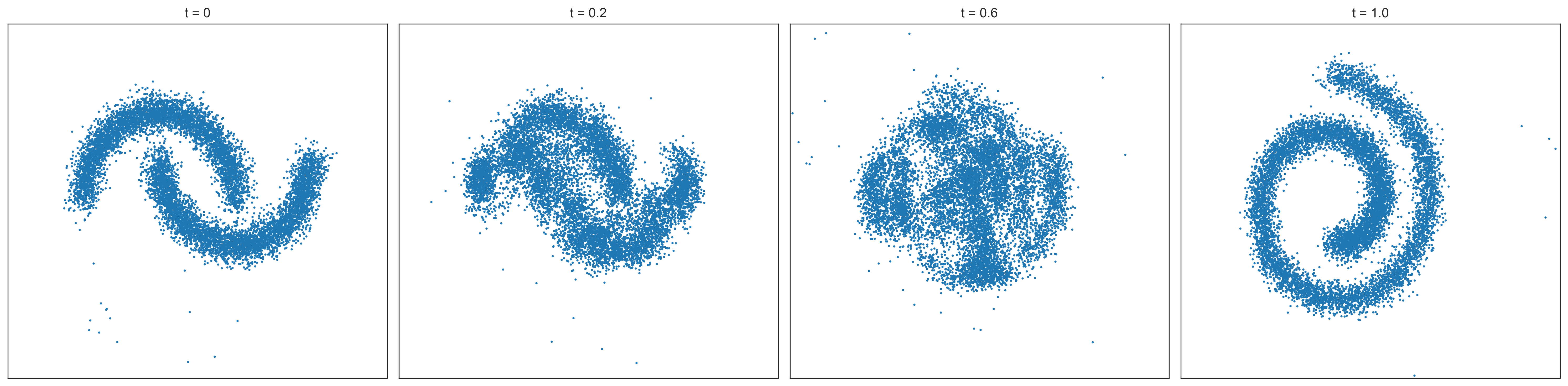}{\centering}
    \phantomsection
    \captionof{figure}{Evolution of a GMM for the moon-to-swiss role dataset with 500 Gaussians.}
    \label{fig:toy}
\end{wraptable}
For simplicity, we omit the mixing term in the Bohm potential \eqref{eq:bohm_mixture}. As shown in Figure \ref{fig:bohm}, this term mainly affects regions far from the primary modes, which are distant from the data distribution. Additionally, its magnitude is generally small compared to the contribution from the individual Gaussian components, resulting in only a minor correction to the overall potential.
While our approach shares similarities with the GMM-based method in \cite{korotin2023light}, it fundamentally describes a different evolutionary dynamic. Figure \ref{fig:toy} shows the learned evolution for a GMM with 500 components, transitioning between the double moons and Swiss roll datasets.

\section{Experiments}
\label{sec: experiments}

\subsection{Single cell population dynamics}

\begin{table}[t!]
\caption{Method performance on the single-cell evolution dataset, measured by EMD-2 distance.}
\centering
\label{ss-stable}
\vskip 0.1in
\begin{small}
\begin{sc}
\begin{adjustbox}{width=0.8\textwidth}
\begin{tabular}{lccccr}
\toprule
Method & $W_2$, $t_1$ & $W_2$, $t_2$ & $W_2$, $t_3$ & $W_2$, $t_4$\\
\midrule
QSBP (ours)   & 0.68 $\pm$ 0.02 & 0.81 $\pm$ 0.03 & 0.85 $\pm$ 0.03 & 0.85  $\pm$ 0.02 \\
OT-Flow       & 0.83 & 1.10 & 1.07 & 1.05 \\
TrajectoryNet & 0.73 & 1.06 & 0.9 & 1.01 \\
IPF           & 0.73 $\pm$ 0.02 & 0.89 $\pm$ 0.03 & 0.84 $\pm$ 0.02 & 0.83 $\pm$ 0.02 \\
SB-FBSDE      & 0.56 $\pm$ 0.01  & 0.80 $\pm$ 0.03 & 1.00 $\pm$ 0.02 & 1.00 $\pm$ 0.01 \\
NLSB          & 0.71 $\pm$ 0.02 & 0.86 $\pm$ 0.03 & 0.83 $\pm$ 0.02 & 0.79 $\pm$ 0.01 \\
NeuralSDE     & 0.69 $\pm$ 0.02 & 0.91 $\pm$ 0.03 & 0.85 $\pm$ 0.03 & 0.81 $\pm$ 0.02 \\
eAM           & 0.58 $\pm$ 0.02 & 0.77 $\pm$ 0.02 & 0.72 $\pm$ 0.01 & 0.74 $\pm$ 0.02 \\
Theoretical minimum   & 0.57  & 0.71 & 0.74  & 0.73 \\
\bottomrule
\end{tabular}
\end{adjustbox}
\end{sc}
\end{small}
\label{table: results}
\end{table}

We apply our method to learn single-cell RNA trajectories from a dataset of human embryonic stem cells \cite{moon2019visualizing} evolving into different cell lineages over a period of 27 days. Single-cell sequencing of the cell population was performed at five different time snapshots (days 0-3, 6-9, 12-15, 18-21, and 24-27). The dataset represents four probability distributions conditioned on these time points. Our goal is to infer evolutionary time trajectories from uncoupled samples of single cells at different times.
We use the experimental setup from \cite{koshizuka2022neural} for train, validation, and test datasets, using the first five principal components as single-cell representations. Data preprocessing is conducted according to \cite{tong2020trajectorynet}.
We compare our method to seven recently developed approaches for inferring population dynamics:
Optimal transport flow \cite{onken2021ot}, TrajectoryNet \cite{tong2020trajectorynet}, Iterative Proportional Fitting based on Schr\"odinger Bridge Problem \cite{de2021diffusion}, Schr\"odinger Bridge solver based on FB-SDE theory \cite{chen2021likelihood}, NLSB method \cite{koshizuka2022neural}, NeuralSDE \cite{li2020scalable}, and action matching method \cite{neklyudov2022action}.
The training protocol and data split are adopted from \cite{koshizuka2022neural}, along with their reported values. Values for the action matching method (eAM) were taken from \cite{neklyudov2022action}, which also followed the same training protocol.
As a metric, we compute the Wasserstein distance between the test marginal distribution and the simulated distribution at time point $t_k$, evaluated 100 times to compute the standard deviation. Simulations were initiated at the previous or next ground truth time points ($t_{k-1}$ or $t_{k+1}$). Results are summarized in Table \ref{table: results}.
Our method demonstrates comparable results to other state-of-the-art approaches, 
and it is very close to the theoretical minimum defined as the Wasserstein distance between the training and test datasets. 

\subsection{Unpaired Image to Image Translation}

For the unpaired image-to-image translation, we adapted the experimental setup from ~\cite{korotin2023light}, which is based on the ALAE dataset \cite{pidhorskyi2020adversarial} and includes the provided encoder and decoder. We used images encoded into the latent space using the Adversarial Latent Autoencoder.
We trained two Gaussian mixture models in the latent space: one with 10 components for images labeled with age $> 18$, and another for images labeled with age $< 18$. The qualitative results of the image translation are shown in Figure ~\ref{fig:deaging}.
\begin{figure}[t!]  
    \centering
    \includegraphics[scale=0.35]{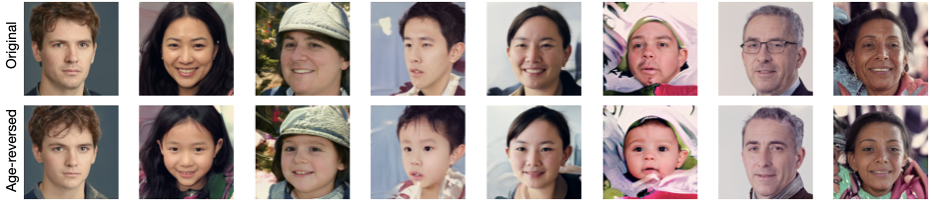}{\centering}
    \caption{Pairs of original images (top row) and corresponding de-aged pairs (bottom row).}
    \label{fig:deaging}
\end{figure}

\subsection{Unpaired Molecular Toxicity Translation in the Tox21 Dataset}

\begin{wraptable}{r}{0.55\textwidth}
\vspace{-1em}
\centering
\caption{Percentage of molecules translated from non-toxic to toxic class.}
\label{table: molecules}
\begin{tabular}{lcccc}
\toprule
\textbf{Toxicity} & \textbf{Class.~F1} & \textbf{S1} & \textbf{S5} & \textbf{S10} \\
\midrule
SR-MMP & 0.67 & 46.3\% & 53.3\% & 54.9\% \\
NR-AhR & 0.58 & 35.8\% & 43.0\% & 44.0\% \\
\bottomrule
\end{tabular}
\vspace{-1em}
\end{wraptable}
Next, we demonstrate unpaired translation of molecules in 512-dimensional CDDD latent space \cite{winter2019learning}. Specifically, we address the task of translating molecular representations from the non-toxic class to the toxic class without requiring paired data. We train a MLP classifier to distinguish between toxic and non-toxic molecules. Using the algorithm described in Algorithm ~\ref{alg: gmm}, we perform probabilistic transport in the latent space between two molecular classes using 30 Gaussian wave packets. 
We focus on two toxicity endpoints (SR-MMP and NR-AhR) from the Tox21 dataset \cite{huang2017editorial}, selected for their abundance of labeled data. The trained classifier is used as an oracle to evaluate toxicity post-translation. We select 1000 non-toxic molecules and generate modified latent samples through the transport procedure. For each source molecule, we generate 1 (S1), 5 (S5) , and 10 (S10) samples, respectively, to evaluate the probability of successful translation into the toxic class. Table \ref{table: molecules} reports the fraction of translated molecules classified as toxic by the oracle. These results suggest that our method enables fast and controllable molecular editing in the absence of paired supervision. See Appendix ~\ref{sec: molecules} for more details.

\subsection{Mean Field Games, Population Dynamics, and Quantum Lagrangian Minimization}

Mean Field Games (MFG) \cite{lasry2006jeux, lasry2007mean} provide a theoretical framework for modeling complex game-theoretic problems involving a large number of interacting agents, where the number of players tends toward infinity. 
The QSBP is closely related to a particular instance of MFG where interactions between agents are important, for example, when a school of fish clusters to enhance collective security. In this case, the dynamics are governed by the Schrödinger equation \cite{swiecicki2016schrodinger}. 
For our mean field experiments, we consider classical crowd navigation environments: the ``S-tunnel'' with asymmetric obstacles, and the ``U-tunnel'' environment with a narrow passage, previously studied in works on MFG \cite{ruthotto2020machine, lin2021alternating, liu2022deep, liu2023generalized, liu2024deep}. In our setup, we model the population evolution between the initial and final distributions using a single Gaussian, evolving through a predefined obstacle configuration. 
We adopt a variational Lagrangian formalism to find the optimal trajectory. We initialize the trajectory using the RRT* path planning algorithm \cite{karaman2011sampling}, which builds a path connecting the start and target distributions. 
The trajectory is modeled by a time-discretized set of Gaussian distributions parametrized by the mean $\bmu(t)$ and variance $\bSigma(t)$ for $t = 0, \dd t, \dotsc, 1$. For simplicity, we consider the case of a diagonal Gaussian.  Given $\bmu( t_i)$ and $\bSigma(t_i)$, the kinetic and potential energies along the trajectory in the diagonal case are computed as
\begin{align}
\label{eq: kinetic_potential_mfg}
\mathcal{K} &=
\sum_{t_i=0}^{1} \left[
\|\dot{\bmu}(t_i)\|^{2}
+\frac14\operatorname{Tr}\, \bigl(\bSigma^{-1}(t_i)\dot{\bSigma}(t_i)\dot{\bSigma}(t_i)\bigr)\right]~,\nonumber\\
\mathcal{U} &= \beta^{2}\, \sum_{t_i=0}^1 \,\operatorname{Tr} \bigl(\bSigma^{-1}(t_i)\bigr)~.
\end{align}
The evolution of individual samples is governed by the stochastic update (see Appendix~\ref{sec: mfg_gaussian})
\begin{equation}
\label{eq: population_update_mfg}
\bx_{i+1} = \bmu(t_i + 1) + \sqrt{1 -2 \beta} \, \bSigma(t_{i+1})^{\frac12}\bSigma(t_i)^{-\frac12} (\bx_i - \bmu(t_i)) + \sqrt{2\beta}\, \dd \bW~,
\end{equation}
ensuring that at each time $t_i$ the sample population has mean $\bmu(t_i)$ and variance $\bSigma(t_i)$.
The final population dynamics trajectory is obtained by minimizing the following Lagrangian objective
\begin{align}
\label{eq: mfg_objective}
\mathcal{L}_{\text{QSB}} &=
\int_{0}^{1}\int_{\mathbb{R}^{n}} 
\Bigl[\|\mathbf{v}(\bx,t)\|^{2} - \|\mathbf{u}(\bx,t)\|^{2}\Bigr]\,
p(\bx,t)\,\mathrm{d}\bx\,\mathrm{d}t \nonumber\\
&= \mathcal{K} - \mathcal{U}~.
\end{align}
We minimize the objective \eqref{eq: mfg_objective} with respect to the parameters $\bmu(t)$ and $\bSigma(t)$ over the time interval $t = 0, \dd t, \dotsc, 1$. 
At each optimization step, we recompute the derivatives 
\begin{align}
\dot{\bmu}(t) & = \frac{\bmu(t+1) - \bmu(t)}{\dd t}~,  &\dot{\bSigma}(t) &= \frac{\bSigma(t+1) - \bSigma(t)}{\dd t}~,
\end{align}
evaluate the updated kinetic and potential energy contributions, and propagate the individual samples according to the stochastic update rule \eqref{eq: population_update_mfg}.

\begin{figure}[t!]  
    \centering
\includegraphics[width=\textwidth]{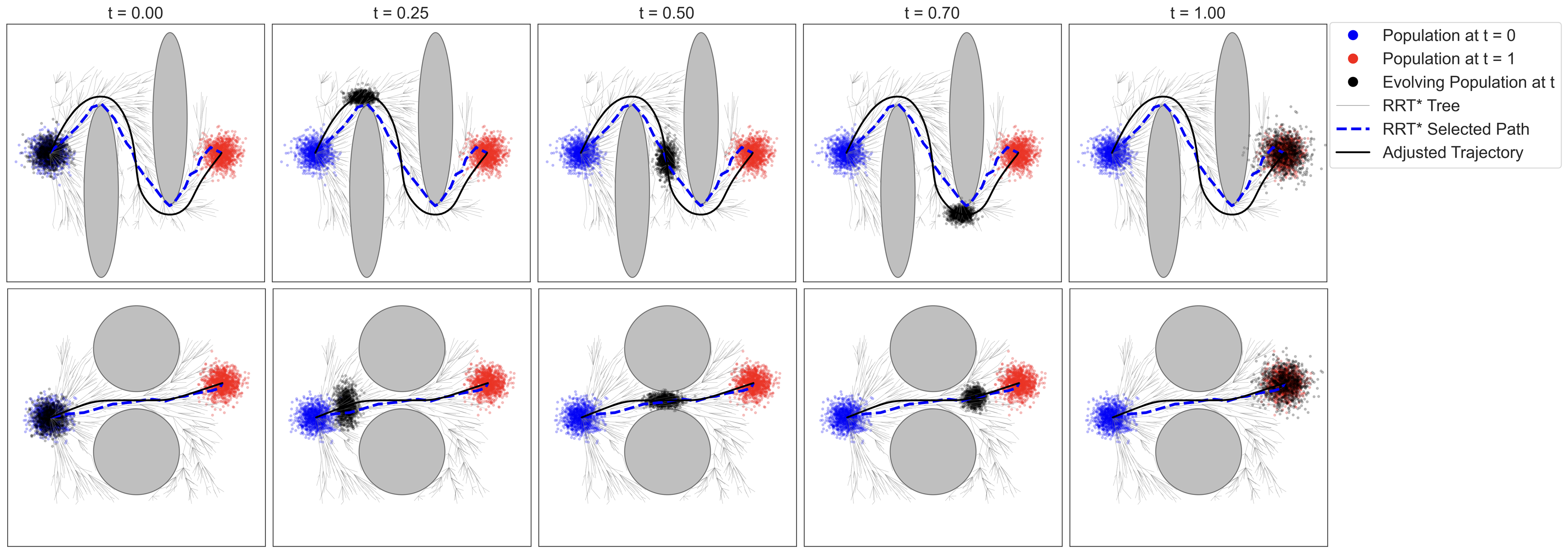}
    \caption{Learned Gaussian evolution dynamics 
    in the S-tunnel (top) and the U-tunnel environment (bottom). The populations at $t = 0,1$ are Gaussian distributions. The RRT* algorithm is used to construct a tree (shown in gray), from which an initial path (dashed blue line) is generated. The solid black line represents the optimal trajectory of the Gaussian mean $\bmu(t)$.}
    \label{fig:mfg}
\end{figure}

\begin{minipage}[t]{0.4\textwidth}
To discourage the population from entering obstacle regions, we include an additional penalty term $\lambda_{\mathrm{obs}} \mathcal{L}_{\mathrm{obs}}$ based on the log-likelihood of samples overlapping with obstacles.
This variational procedure gradually refines the trajectory while maintaining consistency of the population statistics at each time step. The method leverages the RRT* initialization to avoid poor local minima and ensures smooth evolution of the mean and covariance structures. A pseudocode description of the optimization steps is provided on the right.
The resulting learned dynamics between the initial and final distributions are illustrated in Figure~\ref{fig:mfg}, showing the optimized population flows in both the S-tunnel and U-tunnel environments, adjusting their variances and means for optimal passage.
\end{minipage}%
\hfill%
\begin{minipage}[t]{0.57\textwidth}
\vspace{-0.4cm}
\begin{algorithm}[H]
\caption{Variational MFG Trajectory Optimization}
\begin{algorithmic}[1]
\State \textbf{Input:} boundary densities $p_0(\bx),\,p_1(\bx)$, environment, noise level $\beta$, penalty weight $\lambda_{\mathrm{obs}}$
\State \textbf{(i) RRT* warm-start}
    \Statex \quad $\mathcal{P}^{(0)} \gets \texttt{RRT*Path}(p_0, p_1)$
    \Statex \quad Discretize $\mathcal{P}^{(0)}$ into $\{x_t^{(0)}\}_{t=0}^{T}$
\State \textbf{(ii) Initial Gaussian parametrization}
    \Statex \quad $\theta^{(0)} = \{(\bmu_t^{(0)}, \bSigma_t^{(0)})\}_{t=0}^{T}$
\State \textbf{(iii) Trajectory optimization}
\Repeat
    \State Compute energy $\mathcal{K}(\theta)$, $\mathcal{U}(\theta)$, and obstacle penalty $\mathcal{L}_{\mathrm{obs}}(\theta)$
    \State $\nabla_{\theta} \mathcal{L} \gets \partial (\mathcal{L}_{QSB} + \lambda_{\mathrm{obs}} \mathcal{L}_{\mathrm{obs}}) / \partial \theta$
    \State $\theta \gets \theta - \eta \, \nabla_{\theta} \mathcal{L}$
    \State Update population trajectory via \eqref{eq: population_update_mfg}
\Until{convergence}
\State \textbf{Output:} optimal Gaussian trajectory $\theta^\star = \{(\bmu_t^\star, \bSigma_t^\star)\}_{t=0}^{T}$
\end{algorithmic}
\end{algorithm}
\end{minipage}

\section{Discussion and Conclusions}
\label{s:discussion}

Our work represents a step toward greater flexibility in modeling generative processes between two arbitrary distributions. By selecting a different optimality condition, formulated in our framework as a distinct Lagrangian, we derive a bridging problem that follows the Schrödinger equation from quantum mechanics. In the case of Gaussian measures, we provide exact closed-form solutions, effectively placing our models on the same footing as traditional Schrödinger bridge problems in terms of both usability and theoretical understanding. A particularly promising avenue of research is to apply these processes to physical data governed by similar dynamics, such as in molecular dynamics simulations, where the system’s time evolution is described by the Schrödinger equation.

\bibliographystyle{unsrt}  
\bibliography{references}  

\newpage

\appendix

\section{Proofs of Results of Section 2}
\label{app:proof_sec2}

In this appendix we derive the results presented in section \ref{s:classical2quantum} of the main text. 

\begin{proposition}
The solution of the QSBP (Definition \ref{def:QSPB}) is described by the quantum Hamilton-Jacobi equation
\begin{align}
\label{eq:quantumHJ_app}
    \partial_t S(\bx) + \frac12|\nabla S(\bx)|^2 = -Q(\bx)~,
\end{align}
where $Q(\bx)$ is known as the Bohm potential (or quantum potential) and is given as
\begin{align}
\label{eq:bohm_app}
    Q(\bx) = -2 \beta(\bx)^2\frac{\Delta \sqrt{p_t(\bx)}}{\sqrt{p_t(\bx)}}= -\beta(t)^2 (\Delta \log p_t(\bx)+ \frac12 |\nabla\log p_t(\bx)|^2)~.
\end{align}
\end{proposition}
\begin{proof}
As mentioned above, the FPE \eqref{eq:FPE} is valid no matter which Lagrangian is chosen for optimality. 
Therefore, it must be satisfied also in our case. We impose the constraint directly in the Lagrangian through a Lagrange multiplier
\begin{align}
  \cL_{\text{QSB}} &=  \int_\Sigma \left[ \left( \frac12 |\bfb_+(\bx, t)|^2 +  \beta(t) \nabla \cdot \bfb_+(\bx, t)\right) \,  p_t(\bx) + S(\bx, t) \left[ \p_t p_t(\bx)  + \nabla \cdot (\bv(\bx, t) \, p_t(\bx))\right]\right] \dd\bx\, \dd t \nonumber\\
  &=\int_\Sigma \left[  \frac12 |\bfb_+(\bx, t)|^2 -  \beta(t)\bfb_+(\bx, t)\cdot \nabla \log p_t(\bx) - \p_tS(\bx, t)  - \nabla S(\bx, t) \cdot \bv(\bx, t)\right]  p_t(\bx) \dd\bx\, \dd t \nonumber\\
\end{align}
Plugging in the definition of $\bv(\bx, t)$ \eqref{eq:osmotic} and setting to zero the derivative of the above Lagrangian with respect to $\bfb_+$, we get its optimal value 
\begin{align}
    \bfb_+(\bx, t) = \beta(t)\nabla \log p_t(\bx) + \nabla S(\bx, t)~, \qquad \Longrightarrow \qquad \bv(\bx, t) = \nabla S(\bx, t)~.
\end{align}
Substituting the above expressions into the minimization objective, we obtain
\begin{align}
  \cL_{\text{QSB}}
  &=\int_\Sigma \left[  \frac12 \left(\beta(t)^2|\nabla \log p_t(\bx)|^2 + |\nabla S(\bx, t)|^2 + 2\beta(t)\nabla \log p_t(\bx) \cdot \nabla S(\bx, t) \right) \right. \nonumber\\
 & \qquad\quad \left.-  \beta(t)^2|\nabla \log p_t(\bx)|^2 
 -\beta(t)\nabla S(\bx, t) \cdot \nabla \log p_t(\bx)
 - \p_tS(\bx, t)  - |\nabla S(\bx, t)|^2 \right]  p_t(\bx) \dd\bx\, \dd t \nonumber\\
 &=\int_\Sigma \left[ - \frac12\beta(t)^2|\nabla \log p_t(\bx)|^2
 - \p_tS(\bx, t)  - \frac12|\nabla S(\bx, t)|^2 \right]  p_t(\bx) \dd\bx\, \dd t \nonumber\\
    &=\int_\Sigma \left\{\left[  \frac12\beta(t)^2|\nabla \log p_t(\bx)|^2
  - \p_tS(\bx, t)  - \frac12|\nabla S(\bx, t)|^2 \right]  p_t(\bx)
  - \beta(t)^2\nabla p_t(\bx) \cdot \nabla\log p_t(\bx)\right\}\dd\bx\, \dd t \nonumber\\
      &=\int_\Sigma \left[ \beta(t)^2\left( \frac12|\nabla \log p_t(\bx)|^2+  \Delta\log p_t(\bx) \right)
  - \p_tS(\bx, t)  - \frac12|\nabla S(\bx, t)|^2 \right]  p_t(\bx)
\dd\bx\, \dd t~,
\end{align}
which vanishes when
\begin{align}
  \p_tS(\bx, t)  + \frac12|\nabla S(\bx, t)|^2
  =    \beta(t)^2\left( \frac12|\nabla \log p_t(\bx)|^2+  \Delta\log p_t(\bx) \right) = - Q(\bx, t)~,
\end{align}
which concludes the proof.
\end{proof}

\begin{proposition}
The stationary points of the QSBP satisfy the time-dependent Schr{\"o}dinger equation
\begin{align}
\label{eq:schroedinger_app}
    i \p_t\psi(\bx, t) &= -\beta(t) \Delta \psi(\bx, t)   & \psi(\bx)  \psi^*(\bx) &= p_t(\bx)~,
\end{align}
where the wave function $\psi(\bx, t) = \sqrt{p_t(\bx)}e^{\frac{i}{2\beta(t)} S(\bx, t)}$ and phase $S$ are related to the drift velocity $\bv$ via a gradient $\bv = \nabla S$.
\end{proposition}
\begin{proof}
    We want to prove that the function $\psi(\bx, t) = \sqrt{p_t(\bx)}e^{\frac{i}{2\beta(t)} S(\bx, t)}$, where $S(\bx, t)$ satisfies the Schr\"odinger equation. 
Substituting the definition of $\psi$ into the Schr\"odinger equation
\begin{align}
   &i \left(\frac{\p_t p_t(\bx)}{2p_t(\bx)} + \frac{i}{2 \beta(t)} \p_t S(\bx, t)\right) \psi(\bx, t) \nonumber\\
   &\qquad\qquad\qquad= - \beta(t) \left[\frac12 \Delta \log p_t(\bx)  + \frac{i}{2\beta(t)}\Delta S(\bx, t) + \left|\frac12 \nabla \log p_t(\bx) + \frac{i}{2 \beta(t)} \nabla S(\bx, t)\right|^2 \right] \psi(\bx, t)
\end{align}
and dividing by $\psi$ and separating real and imaginary parts we obtain 
\begin{align}
    &\Im:  &\p_t p_t(\bx) 
   &= - \beta(t) \left[\frac{1}{2\beta}\Delta S(\bx, t) + \frac{1}{2 \beta} \nabla \log p_t(\bx) \cdot \nabla S(\bx, t) \right] 2p_t(\bx)\nonumber\\
   &&&= - p_t(\bx)\nabla\cdot \bv(\bx, t) - \nabla p_t(\bx) \cdot \bv(\bx, t)\nonumber\\
   &&&= - \nabla\cdot \left(\bv(\bx, t) p_t(\bx) \right)~,\nonumber\\
&\Re:  & \p_t S(\bx, t)
   &=  \beta(t)^2 \left[ \Delta \log p_t(\bx)   + \frac12\left| \nabla \log p_t(\bx)\right|^2  -\frac{1}{2 \beta(t)^2}\left| \nabla S(\bx, t)\right|^2 \right] \nonumber\\
   &&&= -\frac12\left| \nabla S(\bx, t)\right|^2 - Q(\bx, t)
\end{align}
which are precisely the FPE equation \eqref{eq:FPE} and the quantum Hamilton-Jacobi equation \eqref{eq:quantumHJ}. As these determine the solution of the system, we showed that the same set of solutions are obtained from the Schrödinger equation, concluding the proof.
\end{proof}

\section{Proof of Main Theorem}
\label{app:proof_thorem}
Our main result is given by the following
\begin{theorem}
Given a probability distribution of the form
\begin{align}
    \label{eq:multi_gaussian_t_app}
    p(\bx, t) = \frac{1}{(2\pi)^{n/2} \sqrt{\det \bSigma(t)}} \exp\left(-\frac12 (\bx - \bmu(t))^\top \bSigma^{-1}(t) (\bx(t) - \bmu(t))\right)~,
\end{align}
it solves the QSBP \eqref{eq:QSBP} with boundary conditions 
$\pi_0(\bx) = \cN(\bx_0, \bSigma_0)$ and $\pi_1(\bx) = \cN(\bx_1, \bSigma_1)$
with 
\begin{align}
\label{eq:gauss_solution_app}
\bmu(t)&=\bmu_0+\mathbf{v}_0\,t~,\nonumber\\
\bSigma(t) &= \bSigma_0^{-\frac12}\Bigl[(1-t)\bSigma_0 + t  \Bigl( \bSigma_0^{\frac12}\,\bSigma_1\,\bSigma_0^{\frac12} - \beta^2 \bI \Bigr)^{\frac12}  \Bigr]^2\bSigma_0^{-\frac12} \;+\;  \, t \beta^2 \bSigma_0^{-1}.
\end{align}
\end{theorem}
\begin{proof}
    We starts by considering the constrains imposed by the continuity equation. 
    Then, we will solve the different Hamilton-Jacobi equation.
The continuity equation takes the form
\begin{align}
\label{eq:cont_eq}
    \partial_t p_t(\bx) + \nabla \cdot (\bv(\bx, t) \, p_t(\bx)) = 0~.
\end{align}
We first compute the quantity $\p_tp_t(\bx)$. The pre-factor  $\frac{1}{(2\pi)^{n/2} \sqrt{\det \bSigma (t)}}$  depends on the determinant of the covariance matrix. Taking the time derivative of this term we obtain
\begin{align}
   \frac{\partial}{\partial t} \left( \frac{1}{\sqrt{\det \bSigma (t)}} \right) &= -\frac12 \frac{1}{\left(\det \bSigma(t)\right)^{3/2}} \frac{\partial}{\partial t} \left(\det \bSigma(t)\right) \nonumber\\
    &= -\frac12 \frac{1}{\sqrt{\det \bSigma(t)}} \mathrm{Tr}\left(\bSigma^{-1}(t) \dot{\bSigma}(t)\right)~,
\end{align}
where we used the Jacobi formula for the derivative of the determinant and $\dot{\bSigma}(t) = \p_t\bSigma(t)$.
Taking the time derivative of the exponential term involves differentiating both with respect to  $\bmu(t)$  and  $\bSigma(t)$.
Putting all these components together, the time derivative of  $p_t(\bx)$  is given by
\begin{align}
\p_t p_t(\bx)&=  p_t(\bx) \left[ -\frac12 \mathrm{Tr} \left(\bSigma^{-1}(t) \dot{\bSigma}(t)\right)  +(\bx - \bmu(t))^\top \bSigma^{-1}(t) \dot{\bmu}(t) \right.\nonumber\\
&\qquad\qquad\qquad \left.+ \frac12 (\bx - \bmu(t))^\top \bSigma^{-1}(t) \dot{\bSigma}(t) \bSigma^{-1}(t) (\bx - \bmu(t)) \right]~,
\end{align}
where we used the identity 
$\p_t\bSigma^{-1}(t)= -\bSigma^{-1}(t) \dot{\bSigma}(t) \bSigma^{-1}(t)$ as well as the fact that $\bSigma$ is a symmetric matrix. 
The gradient of $p_t(\bx)$ with respect to $\bx$ is
\begin{align}
\nabla p_t(\bx) = -\bSigma^{-1}(t) (\bx - \bmu(t)) p_t(\bx)~,
\end{align}
and thus the divergence term
\begin{align}
\nabla \cdot \left( p_t(\bx) \bv(\bx, t) \right)  = -\bv(\bx, t)^\top \bSigma^{-1}(t) (\bx - \bmu(t)) p_t(\bx) + p_t(\bx) \nabla \cdot \bv(\bx, t)~.
\end{align}
Plugging everything into \eqref{eq:cont_eq} we obtain
\begin{align}
\label{eq:FPE_intermediate}
 & -\frac12 \mathrm{Tr} \left(\bSigma^{-1}(t) \dot{\bSigma}(t)\right) + (\bx - \bmu(t))^\top \bSigma^{-1}(t) \dot{\bmu}(t) + \frac12 (\bx - \bmu(t))^\top \bSigma^{-1}(t) \dot{\bSigma}(t) \bSigma^{-1}(t) (\bx - \bmu(t)) \nonumber\\
&\qquad\qquad\qquad\qquad\qquad\qquad\qquad\qquad\qquad\qquad\qquad\qquad= (\bx - \bmu(t))^\top \bSigma^{-1}(t) \bv(\bx, t) -  \nabla \cdot \bv(\bx, t)~.
\end{align}
Since both sides of the equations are polynomials in $\bx - \bmu(t)$, we seek a polynomial solution of the form \begin{align}
\bv(\bx, t) = \ba_0(t) + \bOmega_1(t)\left( \bx - \bmu(t)\right)~.
\end{align}

Inserting our ansatz into \eqref{eq:FPE_intermediate} we obtain
\begin{align}
      &-\frac12 \mathrm{Tr} \left(\bSigma^{-1}(t) \dot{\bSigma}(t)\right) + \left( \bx - \bmu(t)\right)^\top \bSigma^{-1}(t) \dot{\bmu}(t) + \frac12 \left( \bx - \bmu(t)\right)^\top \bSigma^{-1}(t) \dot{\bSigma}(t) \bSigma^{-1}(t)\left( \bx - \bmu(t)\right) \nonumber\\
&\qquad\qquad\qquad\qquad= \left( \bx - \bmu(t)\right)^\top \bSigma^{-1}(t) \ba_0(t)+ \left( \bx - \bmu(t)\right)^\top  \bSigma(t)^{-1}\bOmega_1(t)\left( \bx - \bmu(t)\right) -  \mathrm{Tr}\left(\bOmega_1(t)\right)~.
\end{align}
Matching terms with equal degree in $\tbx$ yields
\begin{align}
    &\text{0${}^{th}$-order:} &\frac12 \mathrm{Tr} \left(\bSigma^{-1}(t) \dot{\bSigma}(t)\right) &= \mathrm{Tr}\left(\bOmega_1(t)\right)~,\nonumber\\
    &\text{1${}^{st}$-order:} &\bSigma^{-1}(t) \dot{\bmu}(t) &= \bSigma^{-1}(t) \ba_0(t)~, &\Longrightarrow&  &\ba_0(t)&=\dot{\bmu}(t)~,\nonumber\\
    &\text{2${}^{nd}$-order:} &\frac12 \bSigma^{-1}(t) \dot{\bSigma}(t) \bSigma^{-1}(t) &= \bSigma(t)^{-1}\bOmega_1(t)~, &\Longrightarrow&  &\bOmega_1(t)&=\frac12\left[ \dot{\bSigma}(t) \bSigma^{-1}(t)+\bSigma(t)\bPsi(t)\right]~,
\end{align}
where $\bPsi^\top = -\bPsi$ since $\bSigma(t)^{-1}\bOmega_1(t)$ is defined up to a skew-symmetric component. 

The \text{0${}^{th}$-order} condition is also satisfied since
\begin{align}
    \mathrm{Tr}\left( \bSigma(t)\bPsi(t)\right) = -\mathrm{Tr}\left( \bSigma(t)^\top\bPsi(t)^\top\right) =  -\mathrm{Tr}\left( (\bPsi(t)\bSigma(t))^\top\right) =  -\mathrm{Tr}\left( \bPsi(t)\bSigma(t)\right) =  -\mathrm{Tr}\left( \bSigma(t)\bPsi(t)\right) =0~.
\end{align}
Thus the drift velocity field that satisfies the continuity equation takes the form
\begin{align}
    \label{eq:velocity_gaussian}
    \boxed{
    \bv(\bx, t)=\dot{\bmu}(t) \;+\; \frac12 \left[ \, \dot{\bSigma}(t) \, \bSigma^{-1}(t)+\bSigma(t)\, \bPsi(t)\, \right](\bx-\bmu(t))~.
    }
\end{align}
Now, the condition that \eqref{eq:velocity_gaussian} must be a gradient of a potential, $\bv(\bx, t) = \nabla S(\bx, t)$ implies that the linear term matrix is symmetric, i.e., 

\begin{align}
    \bSigma(t)^{-1}\dot{\bSigma}(t)=
    \dot{\bSigma}(t)\bSigma(t)^{-1}- \bSigma(t)\bPsi(t)- \bPsi(t)\bSigma(t)~.
\end{align}
The potential $S(\bx, t)$ then assumes the form 

\begin{equation}
\label{eq:S_ansatz}
S(\bx, t)=\frac{1}{4}\,\bigl(\bx-\bmu(t)\bigr)^\top \bC(t)\, \bigl(\bx-\bmu(t)\bigr)
+\dot{\bmu}(t)\cdot\bigl(\bx-\bmu(t)\bigr)+f(t)~,
\end{equation}
where 
\begin{equation}
\label{eq:continuity_eq_cond}
\bC(t)=\dot{\bSigma}(t)\,\bSigma(t)^{-1}-\bSigma(t)\bPsi(t)~.
\end{equation}
The quantity \eqref{eq:S_ansatz} must satisfy the Quantum Hamilton-Jacobi equation (QHJE)
\begin{equation}
\label{eq:quantumHJE}
\partial_t S(\bx, t) + \frac12\left|\nabla S(\bx, t)\right|^2 = -Q(\bx, t),
\end{equation}
in the presence of the Bohm potential \eqref{eq:bohm}, which, for a Gaussian distribution, takes the form
\begin{align}
Q(\bx, t)  =\beta(t)^2\,\operatorname{Tr}\!\, \Bigl(\bSigma(t)^{-1}\Bigr) - \frac{\beta(t)^2}{2}\bigl(\bx-\bmu(t)\bigr)^\top \bSigma(t)^{-2} \bigl(\bx-\bmu(t)\bigr)~.
\end{align}
Plugging the quantities
\begin{align}
\partial_t S(\bx, t) &= \frac{1}{4}\,\bigl(\bx-\bmu(t)\bigr)^\top \dot{\bC}(t) \bigl(\bx-\bmu(t)\bigr)
-\frac12\,\bigl(\bx-\bmu(t)\bigr)^\top \bC(t)\,\dot{\bmu}(t)\nonumber\\[1mm]
&\quad +\,\ddot{\bmu}(t)^\top\bigl(\bx-\bmu(t)\bigr)
- \left|\dot{\bmu}(t)\right|^2+\dot{f}(t)~,\nonumber\\
\left|\nabla S(\bx, t)\right|^2
&=\frac{1}{4}\,\bigl(\bx-\bmu(t)\bigr)^\top \bC(t)^2 \bigl(\bx-\bmu(t)\bigr)
+\bigl(\bx-\bmu(t)\bigr)^\top \bC(t)\,\dot{\bmu}(t)+\left|\dot{\bmu}(t)\right|^2~,
\end{align}
into \eqref{eq:quantumHJ} yields
\begin{align}
&\frac{1}{4}\,\bigl(\bx-\bmu(t)\bigr)^\top \Bigl[\dot{\bC}(t)+\frac12\bC(t)^2\Bigr] \bigl(\bx-\bmu(t)\bigr)
+\ddot{\bmu}(t)\cdot\bigl(\bx-\bmu(t)\bigr)
+\Bigl\{\dot{f}(t)-\frac12\left|\dot{\bmu}(t)\right|^2\Bigr\} \nonumber\\
&\qquad\qquad\qquad\qquad\qquad\qquad\qquad\qquad =-\beta^2\operatorname{Tr}\!\Bigl(\bSigma(t)^{-1}\Bigr)
+\frac{\beta^2}{2}\,\bigl(\bx-\bmu(t)\bigr)^\top \bSigma(t)^{-2} \bigl(\bx-\bmu(t)\bigr).
\end{align}

Since this equation must hold for all \(\bx\), we equate the coefficients of the various powers of \((\bx-\bmu(t))\)
\begin{align}
\label{eq:QHJE_coeffs}
    & \text{0${}^{th}$-order:} \qquad &\dot{f}(t)-\frac12\left|\dot{\bmu}(t)\right|^2&=-\beta^2\,\operatorname{Tr}\!\, \Bigl(\bSigma(t)^{-1}\Bigr)~, \nonumber\\
    & \text{1${}^{st}$-order:} \qquad &\ddot{\bmu}(t)&=0\qquad \qquad\Longrightarrow\qquad \qquad  \bmu(t)=\bmu_0+\mathbf{v}_0t~, \nonumber\\
    & \text{2${}^{nd}$-order:}  &\dot{\bC}(t)+\frac12\bC(t)^2&=2\beta^2\,\bSigma(t)^{-2}~.
\end{align}
Given its similarity with the famous Riccati equation, we name the quatratic equation in \eqref{eq:QHJE_coeffs}
the \textit{Quantum Riccati Equation}. This equation can 
be cast into a more familiar form of the Riccati equation if we introduce the following complex matrix 
\begin{equation}
\label{eq:complex_CQ}
    \bC_Q(t) = \bC(t) + 2\, i \, \beta \, \bSigma^{-1} (t)~.
\end{equation}
In terms of \eqref{eq:complex_CQ} the Quantum Riccati Equation takes the form 
\begin{equation}
     \dot{\bC}_Q(t) = -\frac12\bC^2_Q(t)~,
\end{equation}
where we employed the (symmetric part) of the continuity equation condition \eqref{eq:continuity_eq_cond}
\begin{equation}
    \dot{\bSigma}(t) = \frac12 \bigl[\,\bC(t) \bSigma(t) +  \bSigma(t)\, \bC(t)\bigr] \qquad \Longrightarrow \qquad \dot{\bSigma}^{-1}(t) = -\frac12  \bigl[\bC(t)  \bSigma^{-1}(t) +  \bSigma^{-1}(t)\, \bC(t)\bigr]~.
\end{equation}
The continuity equation can also be rewritten in the complex form if we use $\bC_Q(t)$ matrix
\begin{align}
\label{eq: quantum_continuity}
    \dot{\bSigma}(t) &=\frac12  \bigl[\bC(t)  \bSigma(t) +  \bSigma(t)\, \bC(t)\bigr]  =  \frac12  \bigl[\, \bSigma(t)\bC_{\mathrm{Q}}(t) + \bC_{\mathrm{Q}}(t) \bSigma(t)\, \bigr] - 2 i \beta  \bI, 
\end{align}
and the whole problem can be stated as 
\begin{empheq}[box=\fbox]{align}
\label{eq:q_system}
\dot{\bC}_Q(t) &= -\frac{1}{2}\bC_Q^2(t)~, 
&& \text{Quantum Riccati Equation} \\
\dot{\bSigma}(t) &= \frac{1}{2} \left[ \bSigma(t)\, \bC_Q(t) + \bC_Q(t)\, \bSigma(t) \right] - 2 i \beta \bI~, 
&& \text{Symmetric Continuity Equation} \label{eq:symmetric} \\
\bPsi(t) &= \frac{1}{2} \left[ \bC_Q(t)\, \bSigma^{-1}(t) - \bSigma^{-1}(t)\, \bC_Q(t) \right], 
&& \text{Skew-Symmetric Cont.~Equation} \\
\bSigma(0) &= \bSigma_0~, \quad \bSigma(1) = \bSigma_1~, 
&& \text{Boundary Conditions}
\end{empheq}

\paragraph{The Quantum Riccati Equation.} We now tackle the first equation in the above system. 
By employing the substitution
\begin{equation}
    \bK = \bC_Q(t)^{-1}
\end{equation}
the equation simplifies to
\begin{align}
    \dot{\bK}(t) &= \frac12 \mathbb{I}~, &\Longrightarrow&  &\bK &= \frac12\mathbb{I} t + \bK_1~, &\Longrightarrow& &\bC_Q(t) &= \left(\frac12 \mathbb{I} t + \bK_1\right)^{-1} ~.
    \label{eq:K_definition}
\end{align}
Note that since $\bC_Q(t)$ is symmetric, $\bK(t)$ and $\bK_1$ are also symmetric matrices.

\paragraph{The Symmetric Continuity Equation.}
Next, we introduce the variable $\bX(t)=\bK(t)^{-1}\,\bSigma(t)\bK(t)^{-1}$
so that
\begin{equation}
\label{eq:x_def}
\bSigma(t)=\bK(t)\,\bX(t)\,\bK(t)~.
\end{equation}
Differentiating $\bSigma(t)$ we obtain
\begin{align}
\dot{\bSigma}(t)
&=\dot{\bK}(t)\,\bX(t)\,\bK(t)+\bK(t)\,\dot{\bX}(t)\,\bK(t)+\bK(t)\,\bX(t)\,\dot{\bK}(t)\nonumber\\
&=\frac12\,\bX(t)\,\bK(t)+\bK(t)\,\dot{\bX}(t)\,\bK(t)+\frac12\,\bK(t)\,\bX(t)~.
\end{align}
Substituting this expression into \eqref{eq:symmetric}
yields
\begin{equation}
\bK(t)\dot{\bX}(t)\bK(t)= -2i\beta(t) \bI~.
\end{equation}
It follows that
\begin{align}
\dot{\bX}(t)=- 2i\beta\bK^{-2}(t)= -2i\beta \bC_Q^2(t) = 4i \beta \dot{\bC}_Q(t)~,
\end{align}
and integrating
\begin{align}
\bX(t) =4i \beta  \bC_Q(t)  + \bX_1 = 4i \beta\Bigl(\frac{t}{2}\,\bI+\bK_1\Bigr)^{-1} + \bX_1~.
\end{align}
where $\bX_1$ is a constant matrix.
Substituting this into \eqref{eq:x_def} we have
\begin{equation}
\bSigma(t)=\bK(t)\left[ \bX_1 + 4i \beta  \Bigl(\frac{t}{2}\,\bI\;+\;\bK_1\Bigr)^{-1} \right]\bK(t)~.
\end{equation}
At $t=0$ we can solve for $\bX_1$ in terms of the other constants
\begin{equation}
\bX_1 =\bK_1^{-1}\bSigma_0 \bK_1^{-1} - 4i \beta\bK_1^{-1}~,
\end{equation}
which yields the expression
\begin{align}
\bSigma(t)&=\Bigl(\frac12 t\mathbb{I}+\bK_1\Bigr)
\Bigl[\bK_1^{-1}\,\bSigma_0 \,\bK_1^{-1}  -4i \beta\, \bK_1^{-1} 4i \beta \Bigl(\tfrac{t}{2}\,\bI\;+\;\bK_1\Bigr)^{-1}\, \Bigr]
\Bigl(\frac12\,t\,\mathbb{I}+\bK_1\Bigr)\nonumber\\
&=\Bigl(\frac12 t\bK_1^{-1} + \bI\Bigr)
\bSigma_0 \Bigl(\tfrac12 \,t\,\bK_1^{-1} + \bI\Bigr)  - 2i\beta t (\tfrac12 \, t \, \bK_1^{-1} +   \, \bI ).
\end{align}
The second boundary conditions implies
\begin{equation}
    \bSigma_1 = \bOmega\, \bSigma_0 \, \bOmega - 2i\beta  \bOmega~,
\end{equation}
where we defined
\begin{equation}
    \bOmega = \Bigl(\tfrac12 \,\bK_1^{-1} + \bI\Bigr)~.
\end{equation}
To solve the equation above we note that if we multiply from both sides by $\bSigma_0^{\tfrac12}$
\begin{equation}
  \bSigma_0^{\tfrac12} \,  \bSigma_1 \, \bSigma_0^{\tfrac12} = (\bSigma_0^{\tfrac12}\bOmega\,\bSigma_0^{\tfrac12} ) \, (\bSigma_0^{\tfrac12}\bOmega\,\bSigma_0^{\tfrac12} ) - 2i\beta \, (\bSigma_0^{\tfrac12}\bOmega\,\bSigma_0^{\tfrac12} )
\end{equation}
this relations holds only if $\bSigma_0^{\tfrac12}\bOmega\bSigma_0^{\tfrac12}$ and $\bSigma_0^{\tfrac12} \bSigma_1 \bSigma_0^{\tfrac12}$ commute. In this case, the equation can be solved by jointly diagonalizing these matrices, and we
obtain
\begin{equation}
  \bSigma_0^{\tfrac12}\bOmega\,\bSigma_0^{\tfrac12} =  i\beta\,\bI \;\pm \; (-\beta^2\bI + \bSigma_0^{\tfrac12}   \bSigma_1 \bSigma_0^{\tfrac12})^{\tfrac12}~,
\end{equation}
and hence
\begin{equation}
  \bOmega = i \beta\bSigma_0^{-1} \pm \bSigma_0^{-\tfrac12}(-\beta^2\bI + \bSigma_0^{\tfrac12}   \bSigma_1 \bSigma_0^{\tfrac12})^{\tfrac12}\bSigma_0^{-\tfrac12}~.
\end{equation}
The expression for $\bSigma(t)$ then reads
\begin{equation}
\label{eq:sigma_partial_app}
    \bSigma(t)
    \;=\;
    \Bigl[(1-t)\,\bI + t\,\bOmega\Bigr]\,
      \bSigma_{0}\,
    \Bigl[(1-t)\,\mathbf{I} + t\,\bOmega\Bigr]
    -2i\beta\,t\,
    \Bigl[(1-t)\,\bI + t\,\bOmega\Bigr]~,
\end{equation}
where
\begin{align}
  \bOmega
  &=
  \bSigma_{0}^{-\frac12} \bG \bSigma_{0}^{-\frac12}
  +i\beta\,\bSigma_{0}^{-1}~,
  &\bG &= \Bigl(\bSigma_{0}^{\frac12}\,\bSigma_{1}\,\bSigma_{0}^{\frac12} - \beta^{2} \bI \Bigr)^{\frac12}~.
\end{align}
Expanding the terms in \eqref{eq:sigma_partial_app}
yields
\begin{align}    
    \bSigma(t)
    &=
  \bSigma_{0}^{-\frac12}
  \Bigl[(1-t)\bSigma_{0} + t\bG\Bigr]^{2}
  \bSigma_{0}^{-\frac12}
  - t^{2} \beta^{2} \bSigma_{0}^{-1}
  +2it\beta  \left[(1-t)\,\bI + t\,\bSigma_{0}^{-\frac12} \bG \bSigma_{0}^{-\frac12}\right] \nonumber\\
  &\quad - 2i\beta t \left[(1-t)\,\bI + t\,\bSigma_{0}^{-\frac12} \bG \bSigma_{0}^{-\frac12}\right] + 2t^{2} \beta^{2} \bSigma_{0}^{-1} \nonumber\\
 &= \bSigma_{0}^{-\frac12}
  \Bigl[(1-t)\,\bSigma_{0} + t\,\bG\Bigr]^{2}
  \bSigma_{0}^{-\frac12}
  + t \beta^{2} \bSigma_{0}^{-1}.
\end{align}
Thus, we obtain the final form for the time-dependent evolution of the covariance matrix for the Quantum Schrödinger bridge problem
\begin{align}
    \bSigma(t)
    \;=\;
    \bSigma_{0}^{-\frac12}\,
      \Bigl[
        (1-t)\,\bSigma_{0}
        + t\,\Bigl(\bSigma_{0}^{\frac12}\,\bSigma_{1}\,\bSigma_{0}^{\frac12}
                   - \beta^{2} \bI
              \Bigr)^{\frac12}
      \Bigr]^{2}
    \,\bSigma_{0}^{-\frac12}
    +
    t\,\beta^{2}\,\bSigma_{0}^{-1}~,
\end{align}
as claimed.
\end{proof}

\section{Bohm Potential of a Gaussian Mixture}
\label{app:bohm_gaussian}

We consider a Gaussian mixture distribution:
\begin{align}
p(\bx) \;=\; \sum_{k=1}^K \alpha_k \,\cN(\bx; \bmu_k, \bSigma_k),
\end{align}
where \( \alpha_k \ge 0 \), \( \sum_{k=1}^K \alpha_k = 1 \), and
\(\cN(\bx; \bmu_k, \bSigma_k)\) is the \(k\)-th Gaussian component.
We define the \emph{responsibilities} (posterior mixture weights):
\begin{align}
w_k(\bx) \;=\; 
\frac{\alpha_k\,\cN(\bx; \bmu_k, \bSigma_k)}{p(\bx)},
\quad\text{so that}\quad
\sum_{k=1}^K w_k(\bx) \;=\; 1.
\end{align}

\paragraph{Single‐Gaussian Bohm Potential.}
if \(p(\bx) = \cN(\bx; \bmu, \bSigma)\), the
\emph{Bohm potential} is defined as
\begin{align}
Q(\bx)
\;=\;
-\,\beta^2 
\Bigl[\Delta \log p(\bx) 
\;+\;
\tfrac12 \,\bigl\|\nabla \log p(\bx)\bigr\|^2
\Bigr].
\end{align}
For a single Gaussian,
\begin{align}
\nabla \log \cN(\bx; \bmu, \bSigma)
&= -\,\bSigma^{-1}\,(\bx - \bmu)~,
&\Delta \log \cN(\bx; \bmu, \bSigma)
&=\; -\,\mathrm{Tr}(\bSigma^{-1}).
\end{align}
Hence
\begin{align}
Q(\bx)
\;=\;
\beta^2
\Bigl[
\mathrm{Tr}\bigl(\bSigma^{-1}\bigr)
\;-\;
\tfrac12\,(\bx-\bmu)^\top
\bigl(\bSigma^{-1}\bigr)^2
(\bx-\bmu)
\Bigr].
\end{align}
We denote this single‐Gaussian Bohm potential by \(Q_k(\bx)\) for the \(k\)-th component.

\paragraph{Score \(\nabla \log p(\bx)\) of the Mixture.}
For the mixture density \(p(\bx) = \sum_k \alpha_k \,\cN_k(\bx)\), we have
\begin{align}
\log p(\bx)
\;=\;
\log\!\Bigl(\,\sum_{k=1}^K \alpha_k\,\cN_k(\bx)\Bigr),
\quad
\text{where we abbreviate }\cN_k(\bx) := \cN(\bx; \bmu_k, \bSigma_k).
\end{align}
Then
\begin{align}
\nabla \log p(\bx)&=\frac{1}{p(\bx)}\;\nabla \!\Bigl(\sum_{k=1}^K \alpha_k\,\cN_k(\bx)\Bigr) \nonumber\\
&=\frac{1}{p(\bx)}\sum_{k=1}^K \alpha_k \cN_k(\bx)  \frac{\nabla\cN_k(\bx)}{\cN_k(\bx)} \nonumber\\
&=\sum_{k=1}^K w_k(\bx)
\,\nabla \log \cN_k(\bx)\nonumber\\
&=\sum_{k=1}^K 
w_k(\bx)
\bigl[-\,\bSigma_k^{-1}\,(\bx - \bmu_k)\bigr]~.
\end{align}
\paragraph{Score squared of the Mixture.}
\begin{align}
    \bigl\|\nabla \log p(\bx)\bigr\|^2 &= 
    \left(\sum_{j=1}^K w_j(\bx)
\,\nabla \log \cN_j(\bx)\right)^\top \sum_{k=1}^K w_k(\bx)
\,\nabla \log \cN_k(\bx) \nonumber\\
&= \sum_{j,k=1}^K w_j(\bx)w_k(\bx) (\bx - \bmu_j)^\top 
\bSigma_j^{-1}\bSigma_k^{-1}(\bx - \bmu_k)~.
\end{align}

\paragraph{Laplacian \(\Delta \log p(\bx)\) of the Mixture.}
\begin{align}
\Delta \log p(\bx)
\;=\;
\sum_{k=1}^K w_k(\bx)
\,\Delta \log \cN_k(\bx)
\;+\;
\sum_{k=1}^K
\bigl[\nabla w_k(\bx)\bigr]^\top 
\bigl[\nabla \log \cN_k(\bx)\bigr].
\end{align}
Rewriting \(\nabla w_k(\bx)\) in terms of \(\nabla \log p(\bx)\) and \(\nabla \log \cN_k(\bx)\), one arrives at the helpful form:
\begin{align}
\Delta \log p(\bx)
\;=\;
\sum_{k=1}^K w_k(\bx)\,\Delta \log \cN_k(\bx)
\;+\;
\sum_{k=1}^K w_k(\bx)
\,\bigl\|\nabla \log \cN_k(\bx)\bigr\|^2
\;-\;
\bigl\|\nabla \log p(\bx)\bigr\|^2.
\end{align}

\paragraph{Mixture Bohm Potential.}

Using the Laplacian identity above, we find
\begin{align}
\Delta \log p(\bx)
\;+\;
\tfrac12 \,\bigl\|\nabla \log p(\bx)\bigr\|^2
\;=\;
\sum_{k=1}^K w_k(\bx)
\,\Delta \log \cN_k(\bx)
\;+\;
\sum_{k=1}^K w_k(\bx)\,\bigl\|\nabla \log \cN_k(\bx)\bigr\|^2
\;-\;
\tfrac12 \,\bigl\|\nabla \log p(\bx)\bigr\|^2.
\end{align}
Inserting that back into \(Q(\bx)\) and grouping terms in a convenient way yields:
\begin{align}
Q(\bx)
&=\;
-\,\beta^2
\biggl[\;
\sum_{k=1}^K w_k(\bx)
\Bigl(\Delta \log \cN_k(\bx)
     \;+\;
     \bigl\|\nabla \log \cN_k(\bx)\bigr\|^2
\Bigr)
\;-\;
\tfrac12 \,\bigl\|\nabla \log p(\bx)\bigr\|^2
\biggr]\nonumber\\
&=\;
\sum_{k=1}^K w_k(\bx)\;
\underbrace{\bigl[-\,\beta^2(\,\Delta \log \cN_k
    + \tfrac12 \|\nabla \log \cN_k\|^2)\bigr]}_{=\,Q_k(\bx)}
\;+\;
\tfrac{\beta^2}{2}
\Bigl[
   \bigl\|\nabla \log p(\bx)\bigr\|^2
   \;-\;
   \sum_{k=1}^K w_k(\bx)\,\bigl\|\nabla \log \cN_k(\bx)\bigr\|^2
\Bigr].
\end{align}
Hence we arrive at the explicit decomposition:
\begin{align}
\boxed{
Q(\bx)
\;=\;
\sum_{k=1}^K 
w_k(\bx)
\,Q_k(\bx)
\;\;+\;\;
\tfrac{\beta^2}{2}
\Bigl[
  \bigl\|\nabla \log p(\bx)\bigr\|^2
  \;-\;
  \sum_{k=1}^K w_k(\bx)\,\bigl\|\nabla \log \cN_k(\bx)\bigr\|^2
\Bigr].
}
\end{align}
This shows that the Bohm potential of a Gaussian mixture is given by a \emph{responsibility‐weighted} sum of the single‐Gaussian Bohm potentials \(\{Q_k\}\) plus an extra ``coupling term'' that reflects the nontrivial mixture‐log‐density structure.

\section{Molecule Translation in Latent Space}
\label{sec: molecules}
The entire experiment was performed on a standard laptop using only the CPU, without any GPU acceleration.
We use a 512 dimensional latent space representation of 7,831 molecules, each annotated with up to 12 toxicity endpoints. For this experiment, we focus on two endpoints with the most abundant annotations: SR-MMP (918 measurements) and NR-AhR (768 measurements). Each molecule is labeled as toxic or non-toxic, corresponding to class indices 1 and 0, respectively.

\paragraph{Classifier Architecture.}
To distinguish toxic from non-toxic molecules, we train a multitask binary classifier using a fully connected feedforward neural network. The model consists of five linear layers with decreasing hidden dimensions: 256, 128, 64, and 32 units. Each layer is followed by a LeakyReLU activation. The final output layer produces a 12-dimensional logit vector, one per toxicity endpoint. The network is trained using the Adam optimizer (default parameters), with a learning rate of $3 \times 10^{-4}$ and a batch size of 1024. Training is performed on a random train/test split using binary cross-entropy loss until convergence.

\paragraph{Classifier Performance.}
The classifier achieves the following performance on the two target endpoints:

NR-AhR: F1 = 0.584, PR-AUC = 0.689, Balanced Accuracy = 0.723, MCC = 0.530

SR-MMP: F1 = 0.674, PR-AUC = 0.740, Balanced Accuracy = 0.793, MCC = 0.611

\paragraph{Latent Space Translation.}
We apply our GMM-based Schrödinger Bridge model to translate molecules in latent space between the non-toxic and toxic classes. The model is trained with 30 Gaussian wavepackets, using $\beta = 0.01$, a batch size of 10, and a learning rate of $10^{-3}$. Training converges in approximately 10,000 epochs.

Due to the broader distribution of non-toxic molecules in latent space and the concentrated nature of toxic molecules, we focus on forward translation: from non-toxic to toxic. A classification threshold of 0.5 is used to distinguish class membership. For each non-toxic molecule, we generate 1, 5, and 10 samples from the learned transport distribution to evaluate the likelihood of crossing into the toxic region and study saturation effects. An interesting application of our approach would be to combining it with explainability approaches \cite{bertolini2024enhancing, bertolini2023explaining} to gain further insights into the model rationale for generating the chosen trajectory in latent space, which should amount to explain which chemical matter modes are associated a high likelihood in the learned ``toxic'' distribution. 

\section{Mean-Field Games and Lagrangian Minimization}
\label{sec: mfg_gaussian}
Our approach is based on Lagrange minimization principle. Among all admissible trajectories $\bx(t)$ connecting two fixed points $\bx(0) = \bx_0$ and $\bx(1) = \bx_1$ over a fixed time interval $t \in [0, 1]$, the optimal trajectory minimizes the action functional
\begin{align}
\mathcal{A}[\bx(t)] = \int_0^1 \mathcal{L}(\bx(t), \dot{\bx}(t), t) \, \dd t,
\end{align}
where \( \mathcal{L}(x, \dot{\bx}, t) \) is the Lagrangian, typically of the form
\begin{align}
\mathcal{L}(x, \dot{\bx}, t) = \frac{1}{2} \|\dot{\bx}(t)\|^2 - V(\bx(t), t),
\end{align}
with $V(\bx, t)$ denoting a potential energy function. Considering a single gaussian, we can significanly simplify our formulas for kinetic and potential energies, making the algorithm attractive in its simplicity and flexibility.

Let the population density be Gaussian
\begin{align}
p(\bx,t)=\mathcal N\!\bigl(\bmu(t),\bSigma(t)\bigr),\qquad
\;t\in[0,1],
\end{align}
with boundary marginals $p_0(\bx)$ and $p_1(\bx)$ also being Gaussian. Individual samples from the distribution evolve according to the law
\begin{equation}
\label{eq: population_update_mfg_app}
\bx_{i+1} = \bmu(t_i + 1) + \sqrt{1 -2 \beta} \, \bSigma(t_{i+1})^{\frac12}\bSigma(t_i)^{-\frac12} (\bx_i - \bmu(t_i)) + \sqrt{2\beta}\, \dd \bW~,
\end{equation}
ensuring that at each time $t_i$ the sample population has mean $\bmu(t_i)$ and variance $\bSigma(t_i)$.
For the general form of the drift velocity \ref{eq:velocity_gaussian}
with arbitrary antisymmetric matrix $\bPsi(\bx, t)$
\begin{align}
\bv(\bx,t)&=
\dot{\bmu}(t)\, +\, 
\frac12\!\bigl[\dot{\bSigma}(t)\bSigma^{-1}(t)+\bSigma(t)\bPsi(t)\bigr]
\bigl(\bx-\bmu(t)\bigr)~,\nonumber\\
\bu(\bx,t)&=\beta\nabla\!\log p(\bx,t)
          =-\beta\,\bSigma^{-1}(t)\bigl(\bx-\bmu(t)\bigr)~.
\end{align}
As before we introduce the \emph{symmetric} matrix
\begin{align}
\bC(t):=\bigl[\dot{\bSigma}(t)\bSigma^{-1}(t)+\bSigma(t)\bPsi(t)\bigr]
        =\bC^{\!\top}(t),\quad
\by:=\bx-\bmu(t)\sim\mathcal N\!\bigl(0,\bSigma(t)\bigr).
\end{align}
The Quantum Schrödinger Bridge Lagrangian
\begin{align}
\mathcal L_{\mathrm{QSB}}
=
\int_{0}^{1}\!\!\int_{\mathbb{R}^{n}}
\bigl(\|\bv(\bx,t)\|^{2}-\|\bu(\bx,t)\|^{2}\bigr)\,
p(\bx,t)\,\dd\bx\,\dd t
\end{align}
simplifies to the \textbf{kinetic–potential form}
\begin{equation}
\boxed{%
\mathcal L_{\mathrm{QSB}}
=
\int_{0}^{1}
\Bigl[
\underbrace{\|\dot{\bmu}(t)\|^{2}
+\frac14\, \operatorname{Tr}\!\bigl(\bC(t)\bSigma(t)\bC(t)\bigr)}_{\displaystyle K(t)}
-
\underbrace{\beta^{2}\,\operatorname{Tr}\!\bigl(\bSigma^{-1}(t)\bigr)}_{\displaystyle U(t)}
\Bigr]\dd t }.
\label{eq:LQSB_general}
\end{equation}

\paragraph{Diagonal‐\(\bSigma\) special case.}
If \(\bSigma(t)\) is diagonal $\bC(t) = \dot{\bSigma}(t)\bSigma(t)^{-1}$ and one can write down the kinetic energy as 
\begin{align}
K(t) = 
|\dot{\bmu}(t)\|^{2}
+\frac14\, \operatorname{Tr}\!\bigl(\dot{\bSigma}(t) \, \dot{\bSigma}(t)\, \bSigma(t)^{-1}) 
\end{align}
and $\mathcal L_{\mathrm{QSB}}$ reduces to 
\begin{equation}
\mathcal L_{\mathrm{QSB}}
=
\int_{0}^{1}
\Bigl[
\underbrace{\|\dot{\bmu}(t)\|^{2}
+\frac14\, \operatorname{Tr}\!\bigl(\dot{\bSigma}(t)\dot{\bSigma}(t)\,\bSigma(t)^{-1})}_{\displaystyle K(t)}
-
\underbrace{\beta^{2}\,\operatorname{Tr}\!\bigl(\bSigma^{-1}(t)\bigr)}_{\displaystyle U(t)}
\Bigr]\dd t~.
\label{eq:LQSB_general_diag}
\end{equation}

\subsection{Experiment Specifics}
The entire experiment was performed on a standard laptop using only the CPU, without any GPU acceleration.
The objective $\mathcal{L}_{\mathrm{QSB}}$ serves as the loss function in our search for the optimal trajectory. In our two-dimensional "S-tunnel" and "V-neck" experiments, we fix the initial and final means of the population distributions at $t=0$ and $t=1$, respectively. We introduce learnable parameters $\mu(t_i)$ and $\Sigma(t_i)$ for discrete time points $t_i = 0, \dd t, \ldots, 1$. For trajectory learning, we use 100 time steps.

To prevent particles from passing through obstacles, we add a penalization term $\lambda_{\mathrm{obs}} \, \mathcal{L}_{\mathrm{obs}}$, resulting in a total loss function of
\begin{align}
\mathcal{L}_{\mathrm{total}} = \mathcal{L}_{\mathrm{QSB}} + \lambda_{\mathrm{obs}} \, \mathcal{L}_{\mathrm{obs}}~.
\end{align}
The hyperparameter $\lambda_{\mathrm{obs}}$ is chosen so that both terms in the loss are of comparable magnitude, which empirically yields the best performance (in our case $\lambda_{\text{obs}} = 5000$).

In S-tunnel experiment initial population is given by  $\mu(0) = (0, 0)$ and $\mu(1) = (20, 0)$. Two elliptical obstacles are placed at $(6, -4.5)$ and $(14, 4.0)$, each with semi-axes $a = 2.0$ and $b = 10.0$, and the domain is bounded within $[0, 20] \times [-10, 10]$.

In the U-tunnel experiment, the initial population is given by \( \mu(0) = (0, 0) \) and \( \mu(1) = (20, 4) \). Two circular obstacles are placed at \( (10, 8.0) \) and \( (10, -4.0) \), each with semi-axes \( a = 5.0 \) and \( b = 5.0 \), and the domain is bounded within \( [0, 20] \times [-10, 10] \).

A collision-free reference path is generated via RRT*, reparametrized by arc length, and used to initialize the mean trajectory $\bmu(t)$ over $n = 100$ time steps. Both $\bmu(t)$ and $\log \Sigma(t)$ are optimized.

We use $\beta = 0.05$, learning rate $10^{-3}$, AdamW optimizer, and batch size $1000$.

\end{document}